\documentclass[journal]{IEEEtran}
\usepackage{lipsum}
\usepackage{amsfonts}
\usepackage{graphicx}
\usepackage{epstopdf}
\usepackage{algorithmic}
\usepackage{lineno,hyperref}
\usepackage{amsmath,amssymb}
\usepackage{amsthm}
\usepackage{mathtools}
\usepackage{multirow}
\usepackage{latexsym}
\usepackage{cite}
\usepackage{pbalance}

\newtheorem{theo}{Theorem}
\newtheorem{Prop}[theo]{Proposition}

\newtheorem{rem}[theo]{Remark}
\newtheorem{defi}[theo]{Definition}

\newcommand{\intval}[1]{{\bf #1}}
\newcommand{\ul}{\underline}
\newcommand{\ol}{\overline}

\newcommand{\sample}{\mathcal{S}}

\DeclareMathOperator{\DSM}{\text{DSM}}

\ifCLASSINFOpdf
\else
\fi

\begin{document}
\title{Learn and Verify: A Framework for Rigorous Verification of Physics-Informed Neural Networks}

\author{Kazuaki~Tanaka
        and~Kohei~Yatabe,~\IEEEmembership{Member,~IEEE}
        
\thanks{K. Tanaka is with Global Center for Science and Engineering, Waseda University, Tokyo, 169-8555 Japan (e-mail: tanaka@ims.sci.waseda.ac.jp).}
\thanks{K. Yatabe is with Faculty of Engineering, Tokyo University of Agriculture and Technology, Tokyo, 184-8588 Japan (e-mail: yatabe@go.tuat.ac.jp).\vspace{1.5em}}
}

\IEEEoverridecommandlockouts
\IEEEpubid{\makebox[\columnwidth]{\parbox[b]{\columnwidth}{\footnotesize This work has been submitted to the IEEE for possible publication. Copyright may be transferred without notice, after which this version may no longer be accessible.}}\hspace{\columnsep}\makebox[\columnwidth]{ }}

\maketitle

\begin{abstract}
The numerical solution of differential equations using neural networks has become a central topic in scientific computing, with Physics-Informed Neural Networks (PINNs) emerging as a powerful paradigm for both forward and inverse problems. 
However, unlike classical numerical methods that offer established convergence guarantees, neural network-based approximations typically lack rigorous error bounds.
Furthermore, the non-deterministic nature of their optimization makes it difficult to mathematically certify their accuracy.
To address these challenges, we propose a ``Learn and Verify'' framework that provides computable, mathematically rigorous error bounds for the solutions of differential equations.
By combining a novel Doubly Smoothed Maximum (DSM) loss for training with interval arithmetic for verification, we compute rigorous \textit{a posteriori} error bounds as machine-verifiable proofs.
Numerical experiments on nonlinear Ordinary Differential Equations (ODEs), including problems with time-varying coefficients and finite-time blow-up, demonstrate that the proposed framework successfully constructs rigorous enclosures of the true solutions, establishing a foundation for trustworthy scientific machine learning.
\end{abstract}


\IEEEpeerreviewmaketitle

\section{Introduction}
\IEEEPARstart{I}{n} recent years, neural networks have been extensively applied in scientific computing, driven by their remarkable ability to approximate complex functions.
Prominent among these approaches are Physics-Informed Neural Networks (PINNs)~\cite{raissi2019physics}, which have emerged as a powerful framework that directly integrates physical laws into the training process by penalizing the residuals of the governing differential equations within the loss function~\cite{karniadakis2021physics,cuomo2022scientific}.
Beyond PINNs, a wide range of neural network-based solvers has been developed, including the Deep Galerkin and Deep Ritz methods~\cite{sirignano2018dgm,e2018deepritz}, deep Backward Stochastic Differential Equation (BSDE) schemes~\cite{e2017deepbsde,chassagneux2020deepbsde}, and operator learning approaches such as DeepONet~\cite{lu2021deeponet} and Fourier Neural Operators~\cite{li2021fno}.

Despite these successes, a fundamental challenge remains: neural network-based solutions lack rigorous error bounds.
While classical numerical methods, such as finite element or finite difference schemes, offer well-established convergence guarantees, solutions derived via deep learning exhibit stochastic variability and generally lack mathematical certification of accuracy.
Consequently, establishing rigorous error bounds for neural network-based differential equation solvers is a critical prerequisite for ensuring reliability in scientific machine learning.
Although theoretical progress has been made in understanding the behavior of PINNs~\cite{shin2020convergence,molinaro2020pinn,jiao2022rate,masri2025unified,bonito2024consistent}, these results primarily provide asymptotic guarantees or error estimates under idealized conditions, which are often insufficient for practical verification.

In practice, deriving computable, rigorous error bounds for specific numerical solutions is essential. While deterministic \textit{a posteriori} error bounds for PINNs have been proposed~\cite{hillebrecht2022certified,liu2023residual,ahmadova2025lower}, translating these analytic results into fully \textit{machine-verifiable} guarantees remains a challenge. Crucially, the numerical realization of these analytic certificates relies on global quantities such as Lipschitz constants, strong monotonicity constants, residual envelopes, and high-order derivative bounds. In standard implementations, these quantities are often estimated from finite collocation points rather than rigorously enclosed over the continuous domain. Furthermore, evaluating these integral bounds requires strict control over quadrature and floating-point errors, which are typically assumed negligible or handled via heuristic approximations. To elevate these guarantees to the status of machine-verifiable proofs, it is necessary to rigorously account for all sources of theoretical and computational uncertainty.

To address this limitation, we propose a ``Learn and Verify'' framework to rigorously certify the approximate solutions generated by PINNs.
In classical numerical analysis, computer-assisted proofs and verified numerical computations have been extensively developed.
For instance, rigorous bounds for Ordinary Differential Equations (ODEs) are computed using Taylor series expansions and Picard iteration~\cite{lohner1987enclosing, kashiwagi1995power}, while fixed-point theorems have been successfully applied to verify solutions of Partial Differential Equations (PDEs)~\cite{nakao1988numerical, oishi1995numerical, yamamoto1998numerical, day2007validated, plum1991computer, takayasu2013verified}.
Furthermore, software packages for verified numerical computation have been established as fundamental tools in this field~\cite{rump1999intlab,kashiwagi_kv,Nedialkov2010vnode,Kapela2011capd}.
However, a framework for mathematically guaranteeing the reliability of solutions obtained by neural network-based solvers remains unexplored.
This is largely due to the absence of a methodology to bridge the approximate solutions of PINNs with traditional verified numerical computation tools.
To overcome this challenge, we propose constructing sub- and super-solutions using PINNs that rigorously enclose the true solution.

The core idea of the proposed framework is to train PINNs to satisfy specific differential inequalities, ensuring they strictly bound the true solution of the target differential equation from below (as a sub-solution) and above (as a super-solution).
By constructing these sub- and super-solutions using PINNs, we transform the verification problem into a task that is rigorously computable via interval arithmetic.
It is worth noting that this approach is rooted in the classical Intermediate Value Theorem (IVT): for a continuous function $F:\mathbb{R} \rightarrow \mathbb{R}$, if values $\ul{x}$ and $\ol{x}$ satisfy $F(\ul{x})\leq 0$ and $F(\ol{x})\geq 0$, respectively, then the existence of a solution $x$ (where $F(x)=0$) within the interval $[\ul{x},\ol{x}]$ is guaranteed.
In this paper, we extend this principle to ODEs, enabling the use of rigorous interval computation to verify both the existence of the true solution and the accuracy of its approximation.
Importantly, the verification process is independent of the generation method; thus, even non-deterministic approximations produced by ``black-box'' PINNs can be rigorously certified.

In addition to the verification framework, we propose practical learning strategies---namely, Variation Learning and Doubly Smoothed Maximum (DSM) loss---to efficiently construct verifiable candidates using PINNs.
Accordingly, the main contributions of this paper can be summarized as follows:

\begin{enumerate}
    \item
    \textbf{Framework:} We propose a two-phase ``Learn and Verify'' framework for neural differential equation solvers, which upgrades learned approximations into rigorously verified solution enclosures. Rather than relying directly on the neural network output, we construct sub- and super-solutions that provide mathematically rigorous lower and upper bounds for the true solution.
    \item
    \textbf{Learning:} We propose two specific learning strategies. Variation Learning generates structurally valid sub- and super-solutions with explicit control over the enclosure width (which determines the tightness of the error bound). The DSM loss enables robust training of these candidate functions by smoothly approximating the maximum violation of the required differential inequalities. 
    \item
    \textbf{Verification:} We develop the mathematical theory for a rigorous verification pipeline that utilizes interval arithmetic with adaptive subdivision to certify the validity of the constructed sub- and super-solutions. It explicitly accommodates piecewise $C^1$ functions essential for the verification of the existence of global-in-time solutions.
\end{enumerate}

The remainder of this paper is organized as follows. Section~\ref{sec:nn} introduces the preliminaries. Section~\ref{sec:prop} details the proposed framework and provides a step-by-step explanation of the methodology. The theoretical foundation is established in Section~\ref{sec:novelTheorems}. Section~\ref{sec:experiment} demonstrates the effectiveness of the framework through numerical experiments on three nonlinear ODEs. Finally, Section~\ref{sec:conclusion} concludes the paper.

\section{Preliminaries}\label{sec:nn}

\subsection{Deep Neural Network for PINN}

In the proposed framework, the choice of numerical solver is flexible.
Indeed, any black-box function approximator can be used to estimate the solution, provided it satisfies the regularity requirements of the differential equation.
In this paper, for simplicity, we adopt a fully-connected feedforward network with input $t \in [0, T]$ and parameters $\theta = (W^{(n)}, b^{(n)})_{n=1}^N$ defined as follows:
$u_{\theta}(t) = W^{(N)} h^{(N-1)} + b^{(N)}$, $h^{(n)} = g ( W^{(n)} h^{(n-1)} + b^{(n)} )$ $(n = 1, \dots, N-1)$, and $h^{(0)} = t$,
where $W^{(n)}$ and $b^{(n)}$ denote the weight matrix and bias vector of the $n$-th layer, respectively, $N$ is the number of layers, and $g$ is the element-wise non-linear activation function.

While any neural network architecture can be employed, we specifically utilize the SIREN architecture~\cite{sitzmann2020implicit}.
Unlike standard networks, SIRENs employ periodic sinusoidal activation functions.
This choice is motivated by the superior capability of SIRENs to model fine-grained details and accurately represent higher-order derivatives.
Since the loss function in PINNs relies directly on the derivatives of the network output, the smooth, infinitely differentiable nature of the sine function is particularly effective for minimizing differential residuals.

\subsection{Numerical Solution to ODE Approximated by PINN}
Throughout this paper, we focus on the following initial value problem (IVP) for a first-order ODE:
\begin{align}
    \label{eq:ivpode}
    \left\{
    \begin{array}{l l}
        \displaystyle\frac{du}{dt}(t) = f(t,u(t)), & t \in [0,T], \\
        u(0) = a,
    \end{array}
    \right.
\end{align}
where \(T > 0\) is a given time horizon, \(a \in \mathbb{R}\) is the initial value, and \(f: [0,T] \times \mathbb{R} \to \mathbb{R}\) is a continuous function.
%
The goal of a numerical solver is to approximate the solution $u(t)$ over the interval \([0,T]\).
Traditional methods typically achieve this by discretizing the time domain to compute the solution at specific time steps.

In contrast, a PINN approximates the solution $u$ using a neural network $u_\theta$ parametrized by weights $\theta$.
Instead of discrete time-stepping, it learns the solution globally through the following optimization problem:
\begin{align}
    \label{eq:generalLoss}
    \underset{\theta}{\text{Minimize }} \;L_{\text{ODE}}(\theta) + \lambda \, L_{\text{IV}}(\theta),
\end{align}
where the loss functions are defined as
\begin{align}
    L_{\text{ODE}}(\theta) &= \frac{1}{|\mathcal{S}|} \sum_{t \in \sample} \left(\frac{du_\theta}{dt}(t) - f(t,u_\theta(t))\right)^2, \label{eq:Lode}\\
    L_{\text{IV}}(\theta) &= \left(a - u_\theta(0)\right)^2, \label{eq:Liv}
\end{align}
$\lambda > 0$ is a hyperparameter, and $\sample$ denotes a set of sampling points (collocation points) randomly drawn from the domain $[0,T]$.
Note that $\sample$ is typically resampled or varied during the optimization process to train the neural network, ensuring the loss approximates the integral over the entire domain.

By solving the optimization problem \eqref{eq:generalLoss}, the neural network $u_\theta$ aims to approximate the solution to \eqref{eq:ivpode}.
However, the obtained approximations often exhibit significant variance due to the stochastic nature of the training process (e.g., random weight initialization and stochastic collocation sampling).
Consequently, the accuracy of any specific approximation remains uncertain and lacks mathematical certification.

\subsection{Sub- and Super-Solutions of \eqref{eq:ivpode}}

To guarantee the reliability of an approximate solution, sub- and super-solutions are constructed in the proposed framework.
For problem \eqref{eq:ivpode}, they are defined as follows:
	
\begin{defi}
	A pair of functions $\ul{u},\ol{u} \in C^1([0,T])$ satisfying
	\begin{align}
			\label{eq:sub}
			&\left\{\begin{array}{l l}
					\displaystyle\frac{d\ul{u}}{dt}(t) \leq f(t,\ul{u}(t)),~t \in (0,T)\\
					\ul{u}(0) \leq a,\\
				\end{array}\right.
        \\[4pt]
			\label{eq:super}
			&\left\{\begin{array}{l l}
					\displaystyle\frac{d\ol{u}}{dt}(t) \geq f(t,\ol{u}(t)),~t \in (0,T)\\
					\ol{u}(0) \geq a\\
				\end{array}\right.
		\end{align}
	is called a pair of sub-solution $\ul{u}$ and super-solution $\ol{u}$ of \eqref{eq:ivpode}.
\end{defi}

\begin{rem}
	When equality holds in \eqref{eq:sub} or \eqref{eq:super}, $\ul{u},\ol{u}$ coincide with the solution $u$ of \eqref{eq:ivpode}.
    Thus, this definition is a natural and intuitive generalization of the classical solution concept.
    However, we can relax the regularity of $\ul{u},\ol{u}$; their differentiability can be piecewise, as we will discuss later in Section~\ref{sec:novelTheorems}.
\end{rem}

Provided that valid sub- and super-solutions are constructed, the exact solution to \eqref{eq:ivpode} is guaranteed to be enclosed within the interval defined by them.
Furthermore, the tightness of this bounding interval directly indicates the accuracy of the approximate solution.

\subsection{Interval Arithmetic for Verification}\label{subsec:intervalArithmetic}

The verification of sub- and super-solutions requires rigorous evaluation of the differential inequalities in \eqref{eq:sub} and \eqref{eq:super}.
Since standard numerical computations are conducted within the limitations of finite precision arithmetic, the computed results inevitably deviate from the exact values due to rounding errors.
However, by deliberately controlling the rounding mode of the numerical computations, it is possible to obtain evaluations that rigorously bound the true value. 

The IEEE 754 standard, which defines floating-point arithmetic and its basic operations, prescribes four rounding methods: round-to-nearest, truncation, rounding up (towards \(+\infty\)), and rounding down (towards \(-\infty\)).
By utilizing the latter two directed rounding modes, we can enclose arithmetic operations within rigorous bounds.
Specifically, for any real arithmetic operation \(\diamond \in \{+, -, \times, /\}\), we can obtain rigorous upper and lower bounds for the result as follows.

To compute rigorous enclosure, interval arithmetic is used, i.e., simultaneously computing two real numbers to obtain both upper and lower bounds of the true value. 
Let $\intval{x} = [\ul{x},\ol{x}]:=\{x \in \mathbb{R} : \ul{x}\leq x \leq\ol{x} \}$. 
Given intervals \(\intval{x}=[\ul{x},\ol{x}]\) and \(\intval{y}=[\ul{y},\ol{y}]\), the elementary arithmetic operations are defined as
\begin{align*}
	&\intval{x}+\intval{y}=[\ul{x}+\ul{y},~\ol{x}+\ol{y}],\\
	&\intval{x}-\intval{y}=[\ul{x}-\ol{y},~\ol{x}-\ul{y}],\\
	&\intval{x} \times \intval{y} = [\min\{\ul{x}\ul{y},\ul{x}\ol{y},\ol{x}\ul{y},\ol{x}\ol{y}\},\max\{\ul{x}\ul{y},\ul{x}\ol{y},\ol{x}\ul{y},\ol{x}\ol{y}\}],\\
	&\intval{x}/\intval{y} = \intval{x} \times [1/\ol{y},~1/\ul{y}]~~\text{(provided \(0 \notin \intval{y}\))}.
\end{align*}
In computer implementations, when computing the lower (or upper) bound of an interval, downward (or upward) rounding is strictly applied.
This procedure ensures that the resulting machine interval always encloses the true mathematical result, preserving the inclusion property (or rigorous containment) despite the limitations of finite precision arithmetic.

Several software libraries for interval arithmetic have been developed, including C-XSC~\cite{Klatte1993cxsc}, MPFI~\cite{revol2002mpfi}, and kv~\cite{kashiwagi_kv}.
In this paper, we employ INTLAB (INTerval LABoratory)~\cite{rump1999intlab} for all rigorous computations, owing to its rich functionality and ease of use.
For a comprehensive treatment of interval arithmetic, we refer the reader to texts such as \cite{neumaier1990interval,moore2009introduction}.

\section{Proposed Framework: Learn and Verify}\label{sec:prop}
For a given approximate solution to an IVP \eqref{eq:ivpode}, we propose a framework to rigorously verify its accuracy. 
This section focuses on detailing the procedural steps of the proposed framework, while the theoretical foundation is deferred to the subsequent section to maintain clarity and flow.

\begin{figure}
    \centering
    \includegraphics[width=1.0\linewidth]{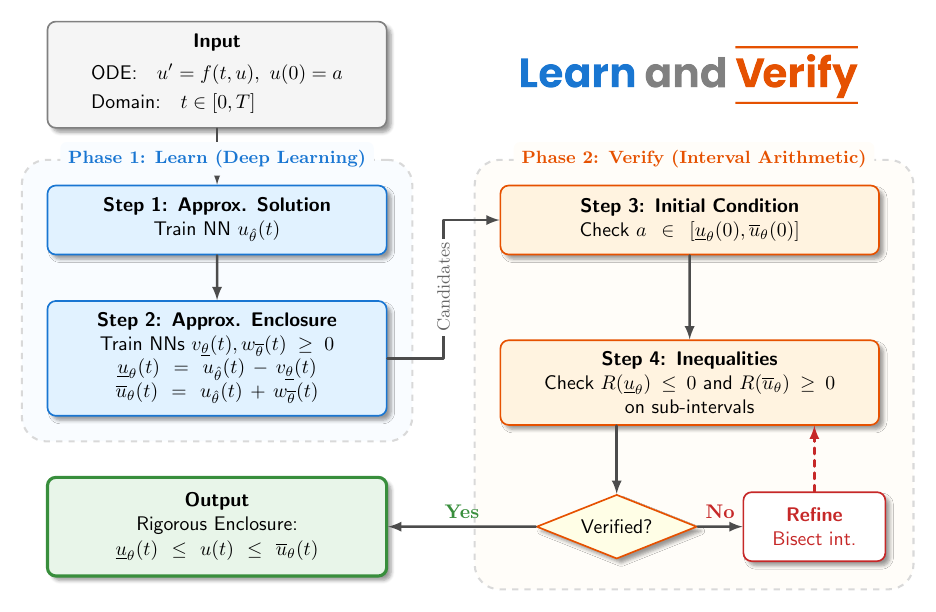}
    \caption{Overview of the proposed framework, ``Learn and Verify.''}
    \label{fig:overview}
\end{figure}

\subsection{Overview of Proposed Framework}

Fig.~\ref{fig:overview} provides a schematic overview of the proposed framework, named ``Learn and Verify.''
It consists of two distinct phases: \textit{Learn phase} and \textit{Verify phase}.
In the Learn phase, three PINNs are trained to construct an approximate solution, along with its corresponding sub- and super-solutions.
Subsequently, the Verify phase rigorously validates the correctness of these enclosing functions.

The \textit{Learn phase} proceeds in two sequential steps to construct the approximate solution and its enclosing bounds.
\begin{itemize}
    \item 
    \textbf{Step 1 (Approximate Solution):} A primary approximate solution $u_\theta(t)$ for IVP \eqref{eq:ivpode} is constructed using standard PINN methodologies. This step is flexible and can employ any existing method.
    \item 
    \textbf{Step 2 (Approximate Enclosure):} Two neural networks with non-negative outputs, $v_{\ul{\theta}}(t)$ and $w_{\ol{\theta}}(t)$, are trained to form the approximate bounds $\ul{u}_\theta(t) = u_{\hat{\theta}}(t) - v_{\ul{\theta}}(t)$ and $\ol{u}_\theta(t) = u_{\hat{\theta}}(t) + w_{\ol{\theta}}(t)$ that satisfy the differential inequalities \eqref{eq:sub} and \eqref{eq:super}, respectively. These candidate functions are validated in the subsequent Verify phase.
\end{itemize}

The \textit{Verify phase} validates the candidate sub- and super-solutions using interval arithmetic, which comprises two steps.
\begin{itemize}
    \item 
    \textbf{Step 3 (Initial Condition Verification):} First, the validity of the initial condition is confirmed. This is a straightforward check performed by evaluating $\underline{u}_\theta(0)$ and $\overline{u}_\theta(0)$ against the initial value $a$.
    \item 
    \textbf{Step 4 (Differential Inequality Verification):} Second, the validity of the differential inequalities is rigorously checked via an adaptive subdivision strategy. The computational domain is partitioned into intervals to bound the residual error. If the verification fails for a specific interval, that interval is further subdivided, and the validation is performed again. This process terminates when the verification succeeds for all intervals, or when a violation of the inequalities is detected within an interval.
\end{itemize}

These four steps are detailed in the following subsections.

\subsection{Step 1: Approximate Solution}\label{subsec:step1}

The first step is the same as the standard PINN methodologies.
Thus, any suitable network architecture may be employed to parameterize the approximate solution $u_\theta(t)$.
In this subsection, we focus on three aspects relevant to ODEs: initial condition, domain sampling, and regularization.

For the initial condition, there are two primary approaches.
The first is the penalty-based approach (soft constraint), where a penalty term $L_{\text{IV}}(\theta) = (a - u_\theta(0))^2$ is added to the cost function.
The second is a hard constraint approach, where the neural network solution is structured to satisfy the initial condition by construction.
This is typically achieved using the form $u_\theta(t) = a + g_\theta(t)$, where $g_\theta(t)$ is a neural network designed such that $g_\theta(0) = 0$ (e.g., $g_\theta(t) = t \cdot h_\theta(t)$).
Both approaches are permissible within the proposed Learn and Verify framework.
Note that the hard constraint approach guarantees the initial value is met by design, thereby allowing us to skip Step 3 (Initial Condition Verification).

For the sampling set $\sample$, a simple approach is uniform random sampling across the entire domain $[0,T]$.
Alternatively, more sophisticated strategies such as region-based (stratified) sampling may be employed, wherein the domain is divided into multiple equal sub-regions and points are sampled uniformly from each of them to ensure balanced coverage.

For regularization, we introduce a penalty term to guide the network toward physically plausible solutions:
\begin{align}
\label{eq:penalty}
L_{\text{Phys}}(\theta) = \max\left(0, \frac{1}{|\sample|}\sum_{t \in \sample} \frac{\partial f}{\partial u}(t,u_\theta(t))\right).
\end{align}
This function penalizes the model when the averaged partial derivative is positive, which corresponds to potentially unstable solution behavior%
\footnote{
The motivation for the penalty function \eqref{eq:penalty} arises from linear stability analysis.
When a solution to the ODE is perturbed as $u(t) + \varepsilon(t)$, the dynamics follow $d(u+\varepsilon)/dt = f(t,u+\varepsilon)$. 
By linearizing the right-hand side as $du/dt+d\varepsilon/dt \approx f(t,u) + (\partial f/\partial u) \, \varepsilon$ and subtracting the original equation $du/dt = f(t,u)$, an ODE for the perturbation $d\varepsilon/dt \approx (\partial f/\partial u) \, \varepsilon$ is obtained.
Since the solution to the standard linear ODE $d\varepsilon/dt = \lambda\varepsilon$ is given by $\varepsilon(t) = \varepsilon(0)\,e^{\lambda t}$, the condition $\partial f/\partial u < 0$ ensures that perturbations decay over time, thereby providing stability to the ODE solution.
}.
This regularization technique can result in dramatic improvements, particularly when optimizing approximate solutions, as will be demonstrated in Section \ref{sec:experiment}.

\subsection{Step 2: Approximate Enclosure}\label{subsec:step2}

In the proposed framework, constructing 
two functions that strictly enclose the approximate solution $u_{\hat{\theta}}(t)$ is essential.
While it is possible to train two independent neural networks, 
guaranteeing the enclosure property (i.e., non-intersection) would require additional penalty terms in the loss function.
To avoid this and to explicitly control the width of the bounds, we propose constructing these bounds as deviations from $u_{\hat{\theta}}(t)$.

Let $v_{\ul{\theta}}(t)$ and $w_{\ol{\theta}}(t)$ denote neural networks constrained to have non-negative outputs.
We propose to construct the candidates of the sub-solution (lower bound) $\ul{u}_\theta$ and super-solution (upper bound) $\ol{u}_\theta$ as deviations from $u_{\hat{\theta}}(t)$:
\begin{align}
    \ul{u}_\theta(t) &= u_{\hat{\theta}}(t) - v_{\ul{\theta}}(t), \\
    \ol{u}_\theta(t) &= u_{\hat{\theta}}(t) + w_{\ol{\theta}}(t).
\end{align}
To enforce non-negativity and control the width of the bounds structurally, we employ a scaled sigmoid function for the output layers of $v_{\ul{\theta}}$ and $w_{\ol{\theta}}$.
Specifically, the final layer is defined as $\varepsilon \sigma(\cdot)$, where $\varepsilon > 0$ is a user-defined tolerance parameter, and $\sigma(\cdot)$ is the sigmoid function.
This architectural choice offers two advantages.
First, it automatically restricts the deviation magnitude to the range $(0, \varepsilon)$, ensuring the bounds remain tight around the approximate solution $u_{\hat{\theta}}(t)$.
Second, it simplifies the optimization process, as non-negativity and range constraints are satisfied by construction. 

The neural networks $v_{\ul{\theta}}(t)$ and $w_{\ol{\theta}}(t)$ are trained such that the approximate sub- and super-solutions, $\ul{u}_\theta$ and $\ol{u}_\theta$, satisfy the differential inequalities defined in \eqref{eq:sub} and \eqref{eq:super}, respectively.
To achieve this, we propose a novel loss function designed based on the following motivation.
From \eqref{eq:sub} and \eqref{eq:super}, the requirements for sub- and super-solutions can be rewritten as
\begin{align} 
    R(\ul{u}_\theta(t)) = \frac{d\ul{u}_\theta}{dt}(t) - f(t,\ul{u}_\theta(t)) &\leq 0, \label{eq:res_sub}\\
    R(\ol{u}_\theta(t)) = \frac{d\ol{u}_\theta}{dt}(t) - f(t,\ol{u}_\theta(t)) &\geq 0, \label{eq:res_super}
\end{align}
where $R(\cdot)$ represents the residual.
Therefore, our objective is to train $\ul{u}_\theta$ (resp.\ $\ol{u}_\theta$) such that the positive part (resp.\ negative part) of the residual is minimized (resp.\ maximized) to zero, thereby ensuring the residual itself become non-positive (resp.\ non-negative).
A naive choice for this objective would be to minimize $\max_t\,(R(\ul{u}_\theta))_+$ (resp.\ $\max_t\,(-R(\ol{u}_\theta))_+$), where $(\cdot)_+ = \max(\cdot,0)$ denotes the positive part. 
However, this function is difficult to minimize due to non-smoothness and the vanishing gradient problem.
Furthermore, it imposes a relatively weak penalty on inequality violations because small violations result in negligible loss values that are easily overshadowed by other terms. 
To circumvent these difficulties, we approximate the objective using smooth functions that provide larger gradients even for small violations.

To realize this, we propose the Doubly Smoothed Maximum (DSM) function defined as follows:
\begin{align}
\label{dsm}
\DSM_{c_1,c_2} \left[ g(t) \right] = c_2 \log \left( \sum_{t \in \mathcal{S}} \left( 1 + \exp\left( \frac{g(t)}{c_1} \right) \right)^{\frac{c_1}{c_2}} \right). 
\end{align}
This serves as a differentiable approximation of $\max_t\,(g(t))_+$ by composing the Softplus and Log-Sum-Exp functions.
Here, $c_1$ acts as the smoothing parameter for the Softplus component; the function $c_1 \log(1 + \exp(x/c_1))$ converges to $(\cdot)_+$ (i.e., ReLU) as $c_1 \rightarrow 0$.
Similarly, $c_2$ controls the approximation accuracy of the Log-Sum-Exp component; the function $c_2 \log(\sum_{x}\exp(x/c_2))$ converges to $\max(\cdot)$ as $c_2 \rightarrow 0$.

Using the DSM functions, we propose the following loss function for training the deviation networks $v_{\ul{\theta}}(t)$ and $w_{\ol{\theta}}(t)$:
\begin{align}
    \label{eq:costfuncsubsuper}
    L_{\text{Sub{\&}Sup}}(\ul{\theta},\ol{\theta}) = &\DSM_{c_1,c_2} \left[ R(u_{\hat{\theta}}(t) - v_{\ul{\theta}}(t)) \right] \nonumber\\
    + &\DSM_{c_1,c_2} \left[ -R(u_{\hat{\theta}}(t) + w_{\ol{\theta}}(t)) \right].
\end{align}
Beyond facilitating optimization through smoothness, this loss function offers a distinct advantage via its inherent soft margin effect.
Unlike the naive formulation, which yields zero gradients upon satisfaction, the Softplus approximation imposes a non-zero penalty even when the inequalities are marginally satisfied.
This characteristic actively drives the optimization to enforce strict inequalities, effectively creating a safety buffer around the solution boundaries.
This buffer is critical for training stability, as stochastic collocation sampling is unlikely to capture the exact location of the maximum residual violation.
Indeed, our extensive preliminary experiments confirmed this necessity; we were unable to obtain valid sub- and super-solutions without the use of this smooth approximation.

It is worth noting that the parameter $c_2$ in DSM is mathematically related to the sample size $|\sample|$. This relationship might provide a theoretical basis for selecting an appropriate value for $c_2$. We provide details of this property in Appendix~\ref{sec:proofDSM}. Additionally, practical application of DSM requires specific implementation strategies to mitigate numerical instability (e.g., arithmetic overflow). The details are given in Appendix~\ref{sec:stableDSM}.

Since the training process relies on minimizing an empirical loss rather than strictly enforcing constraints, the constructed candidates $\ul{u}_\theta$ and $\ol{u}_\theta$ may not fully satisfy the differential inequalities defined in \eqref{eq:sub} and \eqref{eq:super} throughout the entire domain.
Therefore, they must be subjected to rigorous verification in the subsequent steps to certify their validity as strict bounds.

\subsection{Step 3: Initial Condition Verification}

Prior to the rigorous verification, we verify the initial condition by simply checking if the true initial value $a$ is contained in $[\ul{u}_\theta(0),\ol{u}_\theta(0)]$ using the rigorous interval arithmetic.

Note that this step can be skipped if the approximate solution $u_{\hat{\theta}}$ strictly satisfies the initial condition $u_{\hat{\theta}}(0)=a$ (e.g., by using the hard constraint approach). This simplification is a direct benefit of constraining the deviations $v_{\ul{\theta}}$ and $w_{\ol{\theta}}$ to be non-negative, as detailed earlier.

\subsection{Step 4: Differential Inequality Verification}

The final and most critical step is the verification of the candidate sub- and super-solutions.
As will be provided in Section \ref{sec:novelTheorems}, the validity of these functions serves as a sufficient condition for the existence of a true solution to the original problem within the enclosed region.
This subsection focuses exclusively on the computational procedure required to rigorously certify the inequalities, deferring the theoretical proofs to the subsequent section.
Specifically, our goal is to verify that the residuals $R(\ul{u}_\theta(t))$ and $R(\ol{u}_\theta(t))$ strictly satisfy the inequalities \eqref{eq:res_sub} and \eqref{eq:res_super} for all $t \in [0, T]$.

We begin by briefly reviewing the application of interval arithmetic (introduced in Section \ref{subsec:intervalArithmetic}) to the verification of inequalities.
For a given interval $\mathcal{I} = [a,b]$ and a continuous function $g \in C(\mathcal{I})$, our objective is to determine the exact range  $g(\mathcal{I}) = [\inf_{t\in \mathcal{I}}g(t), \sup_{t\in \mathcal{I}}g(t)]$.
However, computing this exact range is often intractable for general nonlinear functions.
Interval arithmetic provides a rigorous methodology to bound this range using finite-precision arithmetic.
By utilizing directed rounding modes during computation, we can calculate an enclosure interval $\mathcal{E} \subset \mathbb{R}$ that is guaranteed to contain the true image, satisfying $g(\mathcal{I}) \subset \mathcal{E}$.
This mechanism enables the rigorous verification of an inequality $g(t) \leq 0$ for all $t\in\mathcal{I}$ by examining the upper bound of the computed enclosure $\mathcal{E}$.
Specifically, the inequality is certified as \textit{satisfied} if $\sup\mathcal{E}\leq0$, and certified as \textit{violated} if $\inf\mathcal{E}>0$.
It is important to note that the enclosure $\mathcal{E}$ derived via interval arithmetic is conservative (i.e., loose), and therefore the indeterminate case where $\inf\mathcal{E}\leq 0 < \sup\mathcal{E}$ may frequently arise.
In such instances, the standard strategy is to subdivide the domain into smaller sub-intervals; this yields a tighter enclosure, eventually allowing for a definitive conclusion.

Building upon this interval arithmetic framework, we propose the following procedure to rigorously validate the candidate sub- and super-solutions.
\begin{enumerate}
    \item 
    The computational domain $[0,T]$ is partitioned into initial sub-intervals $([t_i,t_{i+1}])_{i=0}^{N}$. In practice, we typically employ $N=100$ uniform intervals to balance computational efficiency with verification tightness.
    \item 
    For each sub-interval, rigorous enclosures, $\ul{\mathcal{E}}_i$ and $\ol{\mathcal{E}}_i$, for the ranges of the residuals of the candidate sub- and super-solutions, $R(\ul{u}_\theta([t_i,t_{i+1}]))$ and $R(\ol{u}_\theta([t_i,t_{i+1}]))$, are computed using interval arithmetic. These enclosures satisfy the inclusion property: $R(\ul{u}_\theta([t_i,t_{i+1}])) \subset \ul{\mathcal{E}}_i$ and $R(\ol{u}_\theta([t_i,t_{i+1}])) \subset \ol{\mathcal{E}}_i$.
    \item 
    For each sub-interval, the differential inequalities \eqref{eq:sub} and \eqref{eq:super} are rigorously verified using the computed residual enclosures $\ul{\mathcal{E}}_i$ and $\ol{\mathcal{E}}_i$. The validity of the candidate sub- and super-solutions on the $i$-th interval is determined through the following case analysis:
    \begin{itemize}
        \item 
        \textbf{Sub-solution:} If $\sup\ul{\mathcal{E}}_i \leq 0$, then \textit{valid}. If $\inf\ul{\mathcal{E}}_i > 0$, then \textit{invalid}. If $\inf\ul{\mathcal{E}}_i \leq 0 < \sup\ul{\mathcal{E}}_i$, then \textit{undetermined}.
        \item 
        \textbf{Super-solution:} If $\inf\ol{\mathcal{E}}_i \geq 0$, then \textit{valid}. If $\sup\ol{\mathcal{E}}_i < 0$, then \textit{invalid}. If $\inf\ol{\mathcal{E}}_i < 0 \leq \sup\ol{\mathcal{E}}_i$, then \textit{undetermined}.
    \end{itemize}
    If any sub-interval is classified as \textit{invalid}, the verification process is immediately terminated, and the candidate solution is rejected as it fails to satisfy the corresponding differential inequality.
    \item 
    If all sub-intervals are classified as \textit{valid}, the verification procedure concludes successfully, thereby certifying the approximate solution over the entire domain. If an interval is classified as \textit{undetermined}, it undergoes bisection. The resulting smaller sub-intervals are recursively evaluated. This refinement process continues until all intervals are definitively resolved or a pre-defined maximum recursion depth is reached.
\end{enumerate}
Upon successful validation, the existence of a true solution is rigorously certified within the region enclosed by the sub- and super-solutions. Consequently, these functions establish deterministic, guaranteed error bounds for the approximate solution $u_{\hat{\theta}}$.
Conversely, if the verification process fails, it indicates that either the accuracy of the approximate solution is insufficient or the error tolerance parameter $\varepsilon$ (introduced in Section~\ref{subsec:step2}) is too small (i.e., overly strict).

\subsection{Summary of Training Neural Networks}

Although the preceding subsections have comprehensively detailed the proposed framework, we provide here a concise summary of the neural network training to ensure clarity.

\textit{Step 1 (Approximate Solution):}
In the first step, a standard PINN is trained by minimizing the loss function $L(\theta) = L_{\text{ODE}}(\theta) + \lambda_{\text{IV}} \, L_{\text{IV}}(\theta) + \lambda_{\text{Phys}} \, L_{\text{Phys}}(\theta)$, where the constituent terms are given in \eqref{eq:Lode}, \eqref{eq:Liv} and \eqref{eq:penalty}, respectively. The coefficients $\lambda_{\text{IV}},\lambda_{\text{Phys}}\geq0$ serve as regularization parameters. If the initial condition is structurally satisfied by the network architecture (i.e., the hard constraint approach), the term $L_{\text{IV}}(\theta)$ becomes identically zero and is effectively omitted from the optimization.

\textit{Step 2 (Approximate Enclosure):}
In the second step, two neural networks are trained to represent the deviations. To strictly control the magnitude of these deviations, their output layers utilize a scaled sigmoid function, $\varepsilon\sigma(\cdot)$, scaled by the user-defined error tolerance parameter $\varepsilon>0$. Although a smaller $\varepsilon$ is desirable for precision, it significantly increases the risk of validation failure, as the enclosure becomes tighter and harder to certify. These networks are optimized simultaneously by minimizing the loss function $L_{\text{Sub{\&}Sup}}(\ul{\theta},\ol{\theta})$ defined in \eqref{eq:costfuncsubsuper}, while the approximate solution $u_{\hat{\theta}}$ trained in the first step remains fixed.
The coefficients $c_1,c_2>0$ determine the shape of the loss function; smaller values encourage tighter bounds but may increase the difficulty of optimization, thereby increasing the risk of validation failure.

\section{Theorems for Theoretical Foundation}\label{sec:novelTheorems}

Constructing a rigorous enclosure of the true solution requires a theoretical foundation that guarantees the existence of a solution within a specified region, and computational methods that enable its rigorous verification.
As the computational methodology has been detailed in the preceding section, this section focuses exclusively on establishing the theoretical framework.
Specifically, we present two theorems that provide the mathematical basis for our approach%
\footnote{
To be precise, while the validity of the comparison principle under weaker regularity assumptions is acknowledged in classical literature (see, e.g., \cite{Ladde1985}), explicit proofs for the piecewise $C^1$ setting are not readily available.
We therefore state and prove this result in Theorem~\ref{cor:ode-localsol} for the sake of completeness.
In contrast, Theorem~\ref{theo:global} represents a novel contribution of this paper.
}.

\subsection{Local-in-time Solutions (Finite Time Horizon $[0,T]$)}
First, we establish the precise conditions under which the true solution is strictly enclosed by the sub- and super-solutions over a finite time horizon $[0, T]$.
The following theorem offers a generalization of the classical theory by admitting piecewise $C^1$ functions.
This extension is significant from both theoretical and practical perspectives.

\begin{theo}
	\label{cor:ode-localsol}
	Given partition points \( 0 = t_0 < t_1 < \cdots < t_n < t_{n+1} = T \).
	Let \(\ul{u}, \ol{u} \in \bigcap_{i=0}^{n} C^1([t_i,t_{i+1}])\) such that
	\begin{align}
		\label{eq:ivpode-sub-multi}
		&\left\{\begin{array}{l l}
			\displaystyle\frac{d\ul{u}}{dt}(t) \leq f(t,\ul{u}(t)), &t \in \displaystyle\bigcup_{i=0}^{n} [t_i,t_{i+1}],\\
			\ul{u}(0) \leq a,\\
		\end{array}\right.
    \\[4pt]
		\label{eq:ivpode-super-multi}
		&\left\{\begin{array}{l l}
			\displaystyle\frac{d\ol{u}}{dt}(t) \geq f(t,\ol{u}(t)), &t \in \displaystyle\bigcup_{i=0}^{n} [t_i,t_{i+1}],\\
			\ol{u}(0) \geq a.\\
		\end{array}\right.
	\end{align}
	Here, within each interval, the right-hand derivative is taken at the left endpoint, and the left-hand derivative is taken at the right endpoint.
	Assuming that \(\ul{u}(t) \leq \ol{u}(t)\) for all \( t \in [0,T] \) and that \( f \) is continuous over the set \(\{(t,u)~:~\ul{u}(t) \leq u \leq \ol{u}(t),~t \in [0,T]\}\), there exists a solution \( u \in C([0,T]) \) to \eqref{eq:ivpode} with the property that \(\ul{u}(t) \leq u(t) \leq \ol{u}(t)\) for every \( t \in [0,T] \).
\end{theo}
%

\begin{proof}

We define the truncation function $\mathcal{T}u(t)$ as follows:
\begin{align}
	\mathcal{T} u(t):=\left\{\begin{array}{lll}
		\overline{u}(t) & \text { if } \quad u(t)>\overline{u}(t), \\
		u(t) & \text { if } \quad \underline{u}(t) \leq u(t) \leq \overline{u}(t), \\
		\underline{u}(t) & \text { if } \quad u(t)<\underline{u}(t).
	\end{array}\right.
\end{align}
Since \( f(t,\mathcal{T}u(t)) \) is bounded, a solution \( u \) of the IVP exists:
\begin{align}
	\left\{\begin{array}{l l}
			\displaystyle\frac{du}{dt}(t)=f(t,\mathcal{T}u(t)), &t \in [0,T],\\
			u(0)=a.\\
	\end{array}\right.
\end{align}
For \( \varepsilon>0 \), define \( \overline{u}_{\varepsilon}(t):=\overline{u}(t)+\varepsilon(1+t) \) and \( \underline{u}_{\varepsilon}(t):=\underline{u}(t)-\varepsilon(1+t) \), making it clear that \( \underline{u}_{\varepsilon}(0)<u(0)<\overline{u}_{\varepsilon}(0) \). Suppose there exists \( t_1\in(0,T] \) where \( \overline{u}_{\varepsilon} \) and \( u \) first intersect. That is, \( u(t)<\overline{u}_{\varepsilon}(t) \) is satisfied for \( t\in[0,t_1) \), and for the first time at \( t=t_1 \), \( u \) penetrates \( \overline{u}_{\varepsilon} \) from below. In this case, we have \( \mathcal{T}(u(t_1))=\overline{u}(t_1) \) because \( u(t_1)=\overline{u}_{\varepsilon}(t_1)>\overline{u}(t_1) \). Therefore,
\[
\frac{d\overline{u}_{\varepsilon}}{dt} (t_1) > \frac{d\overline{u}}{dt} (t_1) \geq f(t_1,\overline{u}(t_1)) = f(t_1,\mathcal{T}(u(t_1))) = \frac{du}{dt}(t_1).
\]
On the other hand, since \( u(t)<\overline{u}_{\varepsilon}(t) \) for \( t\in[0,t_1] \), for a sufficiently small \( h>0 \),
\begin{align*}
\frac{\overline{u}_{\varepsilon}(t_{1})-\overline{u}_{\varepsilon}(t_{1}-h)}{h}<\frac{u(t_{1})-u(t_{1}-h)}{h}
\end{align*}
holds true. Hence,
\begin{align*}
\frac{d\overline{u}_{\varepsilon}}{dt} (t_1) 
&= \lim_{h\rightarrow+0}\frac{\overline{u}_{\varepsilon}(t_{1})-\overline{u}_{\varepsilon}(t_{1}-h)}{h} \\
&\leq \lim_{h\rightarrow+0}\frac{u(t_{1})-u(t_{1}-h)}{h} = \frac{du}{dt}(t_1).
\end{align*}
This results in a contradiction, indicating that no such \( t_1 \) exists. 
By a similar argument for \( \underline{u}_{\varepsilon}(t) \), it can be shown that \( \underline{u}_{\varepsilon}(t)<u(t) \), and therefore it must be the case that \( \underline{u}_{\varepsilon}(t)<u(t)<\overline{u}_{\varepsilon}(t) \) for all \( t\in[0,T] \).
By taking the limit as \( \varepsilon \) approaches 0, it is shown that \( \underline{u}(t)\leq u(t)\leq \overline{u}(t) \) for all \( t\in[0,T] \).

\end{proof}

\begin{rem}
The piecewise $C^1$ condition in Theorem \ref{cor:ode-localsol} is particularly suitable for two reasons. First, from a theoretical perspective, this relaxed regularity requirement enables the construction of global-in-time solutions by extending time intervals, as will be demonstrated in the subsequent theorem (Theorem \ref{theo:global}). When extending sub- and super-solutions beyond a given time horizon, the connection points may exhibit discontinuous derivatives, necessitating the piecewise differentiability framework. Second, from a practical standpoint, this condition accommodates neural networks with non-smooth activation functions, such as ReLU, which are differentiable almost everywhere but exhibit points of non-differentiability. This flexibility broadens the applicability of our verification framework to a wider range of neural network architectures.
\end{rem}

\subsection{Global-in-time Solutions (Infinite Time Horizon $[0,\infty)$)}

Second, we extend the solution $u$ from a local interval to the unbounded global domain $[0, \infty)$, while ensuring it remains strictly enclosed within the sub- and super-solutions.
This step establishes the existence of a global-in-time solution.

\begin{theo}
    \label{theo:global}
    In addition to the assumption in Theorem \ref{cor:ode-localsol}, assume, for all \( t \in (T,\infty) \), that
    \begin{align*}
    &0\leq f(T,\ul{u}(T)) \leq f(t,\ul{u}(t)),\\
    &0\geq f(T,\ol{u}(T)) \geq f(t,\ol{u}(t)).
    \end{align*}
    Then there exists a solution \( u \in C([0,\infty)) \) to
    \begin{align}
        \label{eq:ivpode-global}
        \left\{\begin{array}{l l}
            \displaystyle\frac{du}{dt}(t)=f(t,u(t)), &t \in (0,\infty),\\
            u(0)=a\\
        \end{array}\right.
    \end{align}
    such that \( \ul{u}(t) \leq u(t) \leq \ol{u}(t) \) for all \( t \in [0,\infty) \).
\end{theo}
\begin{proof}
    Theorem \ref{cor:ode-localsol} guarantees the existence of a solution \( u \in C([0,T]) \) to \eqref{eq:ivpode} that satisfies \( \ul{u}(t) \leq u(t) \leq \ol{u}(t) \) for all \( t \in [0,T] \).
    It remains to demonstrate that this solution can be extended to \( u \in C([0,\infty)) \).
    
    Let \( \tau>T \) be arbitrarily fixed.
    We extend \( \ul{u},\ol{u} \) to \( [T,\tau) \)
    by defining
    \( \ul{u}(t)=\ul{u}(T) \) and \( \ol{u}(t)=\ol{u}(T) \) for all \( t \in [T,\tau) \).
    Owing to \eqref{eq:ivpode-global}, the extended functions \( \ul{u} \) and \( \ol{u} \) satisfy
    \begin{align*}
        &\frac{d\ul{u}}{dt}(t) = 0 \leq f(T,\ul{u}(T)) \leq f(t,\ul{u}(t)),~~t \in (T,\tau),\\
        &\frac{d\ol{u}}{dt}(t) = 0 \geq f(T,\ol{u}(T)) \geq f(t,\ol{u}(t)),~~t \in (T,\tau),\\
        &\ul{u}(t) \leq \ol{u}(t),~~t \in [T,\tau].
    \end{align*}
    Therefore, employing Theorem \ref{cor:ode-localsol} for \( 0=t_0<t_1<\cdots<t_n(=T)<t_{n+1}(=\tau) \), we ascertain that there exists a solution \( u \in C([0,\tau]) \) to \eqref{eq:ivpode} such that \( \ul{u}(t) \leq u(t) \leq \ol{u}(t) \) for all \( t \in [0,\tau] \).
    Since the extended functions \( \ul{u} \) and \( \ol{u} \) are constant on \( [T,\tau) \), their derivatives may be discontinuous at the connection point \( t=T \), but this is permitted by the piecewise \( C^1 \) condition.
    Since \( \tau>T \) was chosen arbitrarily, the desired assertion follows.
\end{proof}

\section{Numerical Examples}\label{sec:experiment}

To demonstrate the capabilities and versatility of the proposed framework, we applied it to three ODEs. The first is the logistic equation, a classical problem where the analytic solution is known. This case serves to validate the quantitative accuracy of the proposed error bounds against the ground truth. 
The second is a generalized logistic equation, a variant featuring time-varying coefficients, for which an analytic solution exists, but it cannot be written explicitly using elementary functions.
This demonstrates the applicability of the proposed framework to non-trivial problems where exact verification is impossible via traditional analytical means. The third is the Riccati equation, a nonlinear equation whose solution diverges to infinity within a finite time. This example illustrates the ability of the proposed framework to rigorously enclose the blow-up time (singularity) of the solution.

\subsection{Common Experimental Setup}

All numerical experiments were conducted on a workstation%
\footnote{Regarding computational time, the training of the approximate solution (Step 1) typically required between 30 seconds and 5 minutes. The subsequent construction of the sub- and super-solutions (Step 2) took approximately 3 to 15 minutes, depending on problem complexity and the number of epochs. In contrast, the verification using interval arithmetic (Step 4) is computationally inexpensive, typically completing within a few seconds. Note that these timings were measured in a CPU environment. GPU acceleration is expected to significantly reduce the computational time in the training phase.}
with Windows~11 Pro (version 24H2), an AMD Ryzen~9 9950X 16-Core Processor (4.30\:GHz), and 128\:GB RAM (3600\:MT/s), using MATLAB 2025b and INTLAB version 14. For reproducibility, the source code, trained model weights, and verification scripts are made publicly available at
\url{https://github.com/Kazuaki-Tanaka/learn-and-verify-ode}.

We employed SIREN 
consisting of 4 hidden layers with 30 neurons per layer. For Step 2, we additionally incorporated a scaled sigmoid function at the output layer.
They are initialized as for SIREN~\cite{sitzmann2020implicit}, and the training was performed using the Adam optimizer.
The mini-batch size was set to 128 for Step 1 and increased to 1280 for Step 2, utilizing region-based random sampling with 100 equal regions.
The parameters for $L_{\text{Sub{\&}Sup}}$ defined in \eqref{eq:costfuncsubsuper} were set to $c_1 = 10^{-2}$ and $c_2 = 10^{-3}$.

\subsection{Classical Logistic Equation}

We first consider the classical logistic equation, a fundamental model for population growth that incorporates the effect of limited resources through a carrying capacity. The dynamics are governed by the following initial value problem:
\begin{align}
	\label{eq:logistic}
	\left\{\begin{array}{l l}
		\displaystyle\frac{du}{dt}=f(t,u):=ru\left(1-\frac{u}{k}\right), &t \in (0,T),\\
		u(0)=a\\
	\end{array}\right.
\end{align}
where $u = u(t)$ represents the population at time $t$, $r$ is the intrinsic growth rate, $k$ is the carrying capacity, and $a$ is the initial population size.
Its analytical solution is given as
\begin{align}
u(t) = \frac{k}{1 + \left(\frac{k}{a} - 1\right) e^{-rt}}.
\end{align}
The existence of this exact solution enables us to rigorously quantify the accuracy of our learned approximate solution and the tightness of the derived error bounds. In our numerical experiments, we adopted the following parameter set: $r=1$, $k=2$, $a=0.5$, and $T=10$.

\begin{figure}[t]
    \centering
    \hspace*{-0.5cm}
    \begin{tabular}{cc}
        \includegraphics[width=0.24\textwidth]{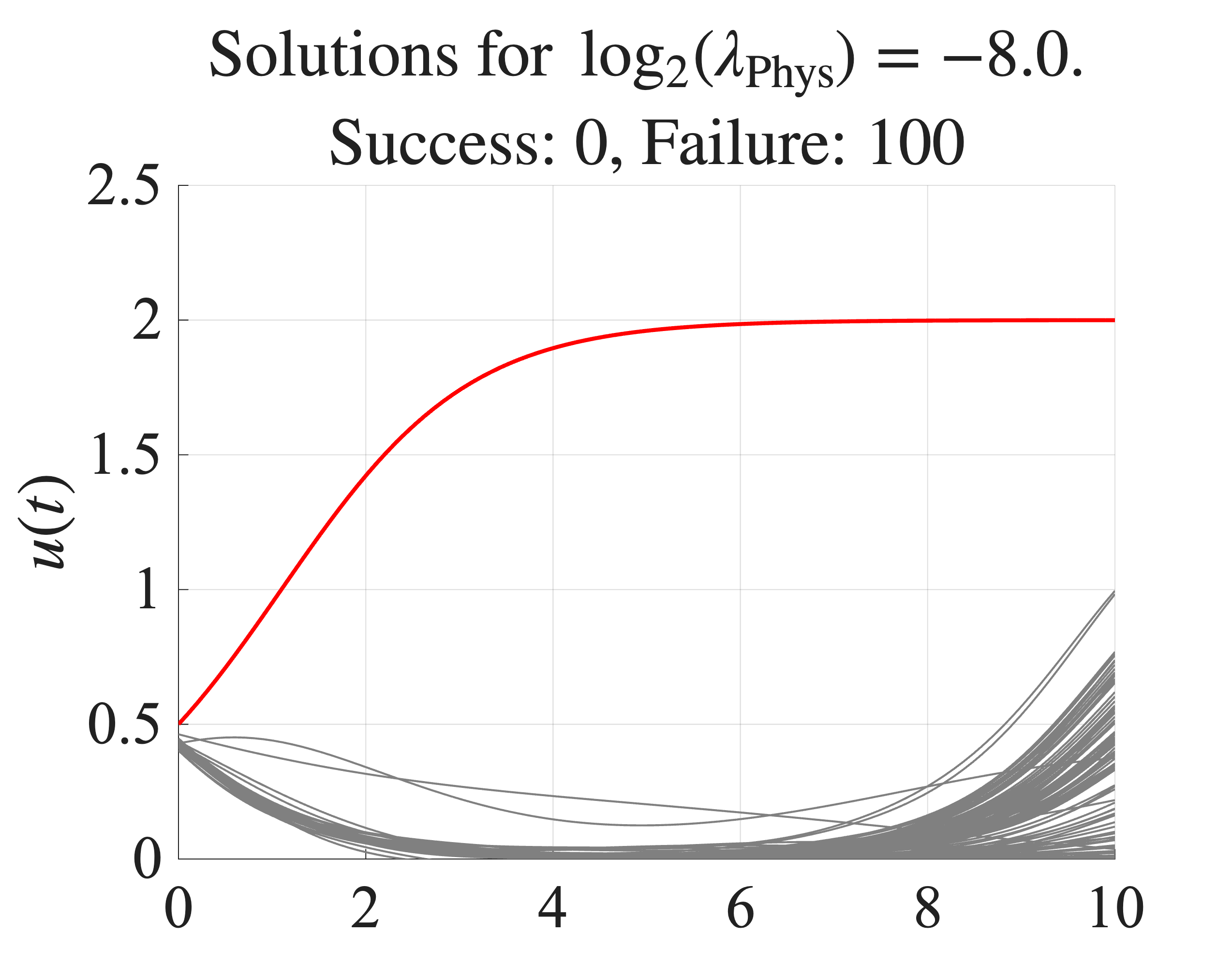} &
        \includegraphics[width=0.24\textwidth]{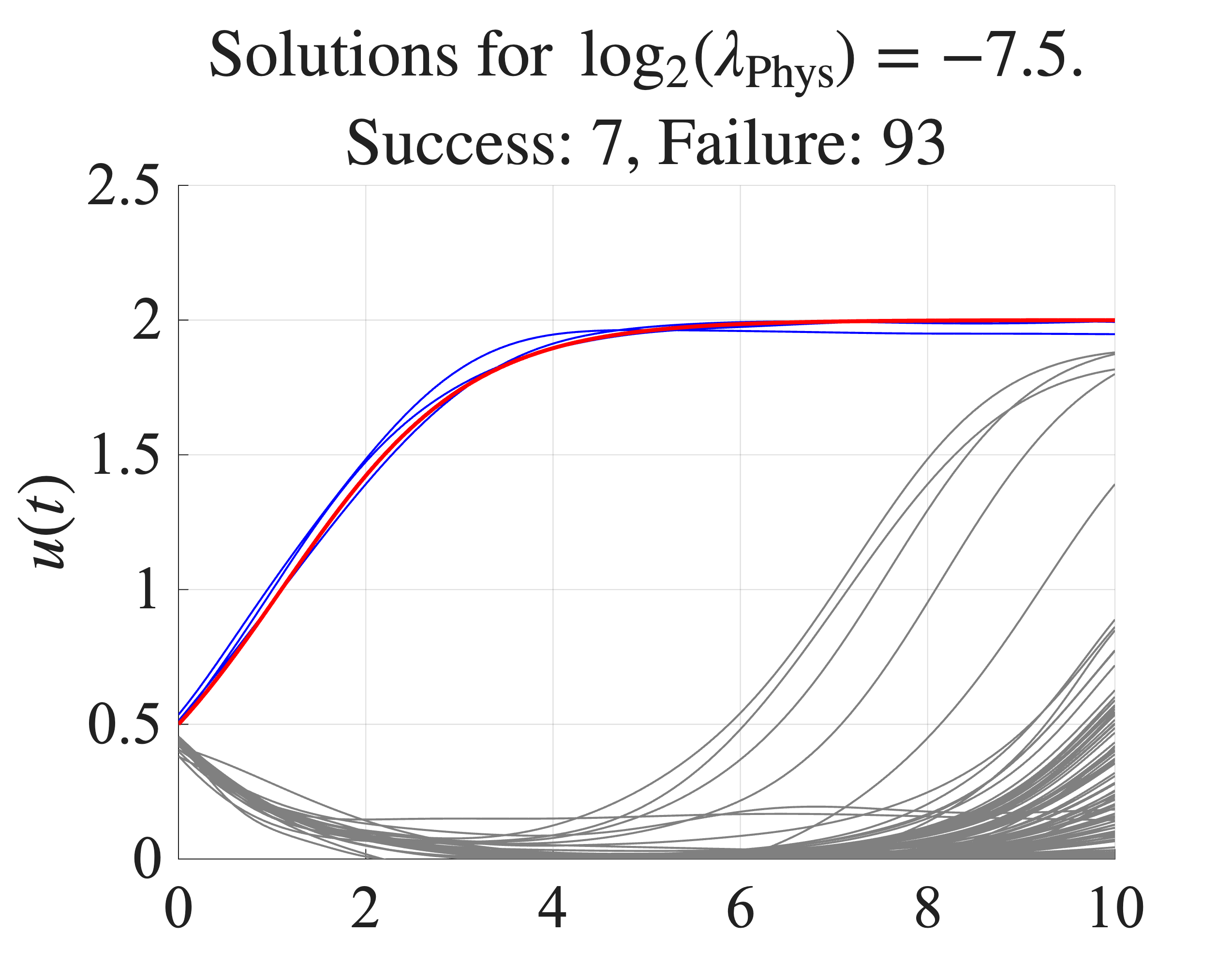} \\
        
        \includegraphics[width=0.24\textwidth]{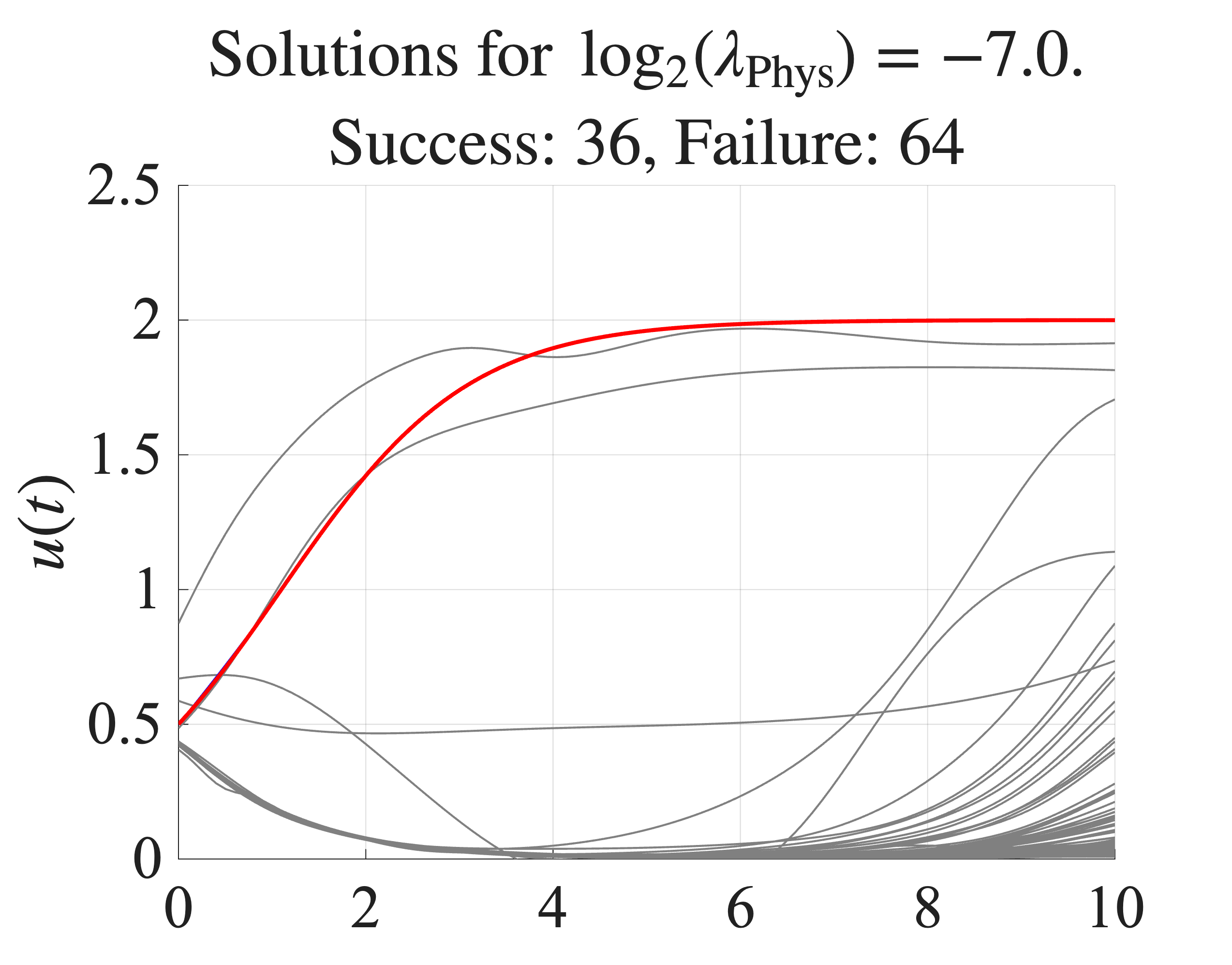} &
        \includegraphics[width=0.24\textwidth]{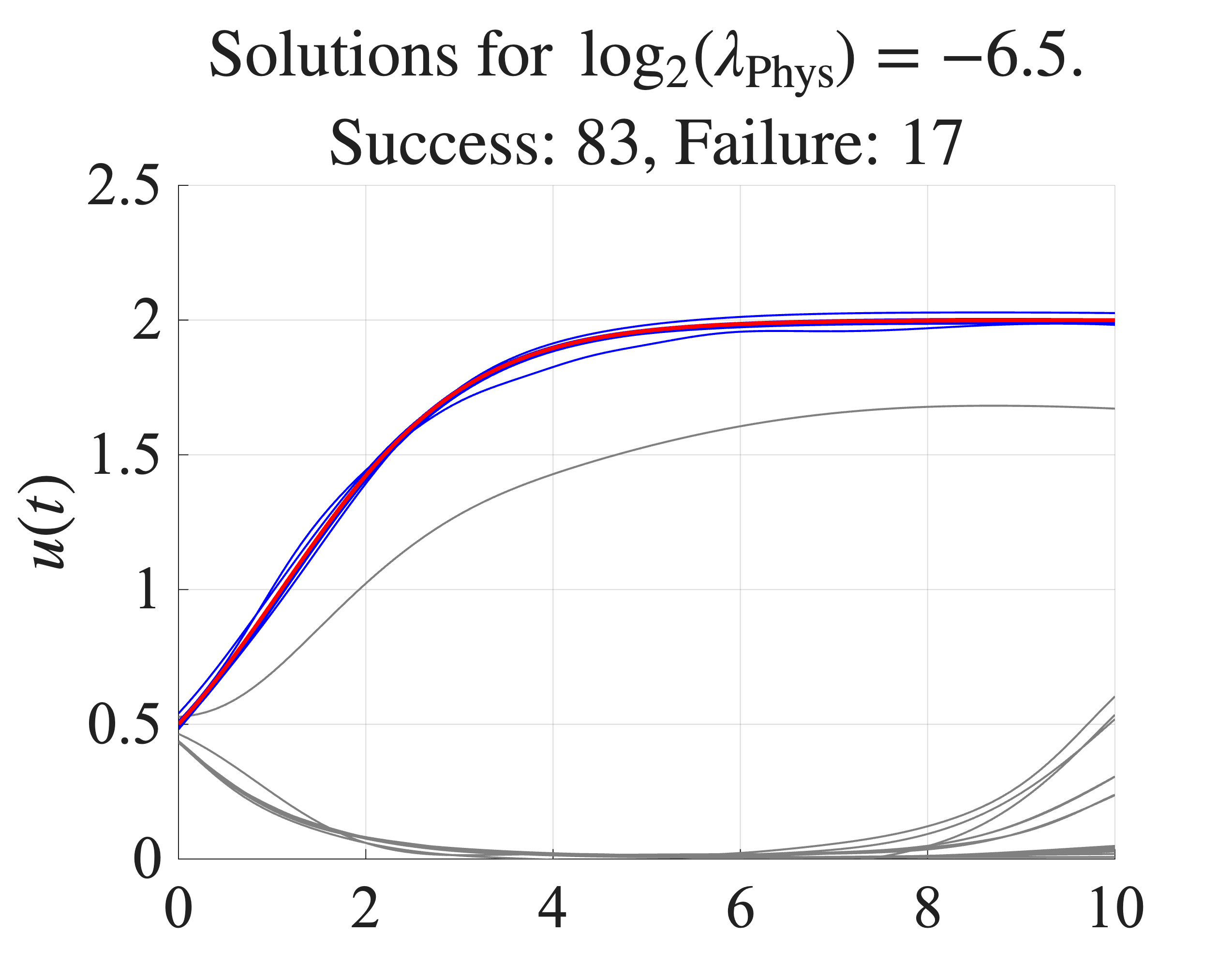} \\
        
        \includegraphics[width=0.24\textwidth]{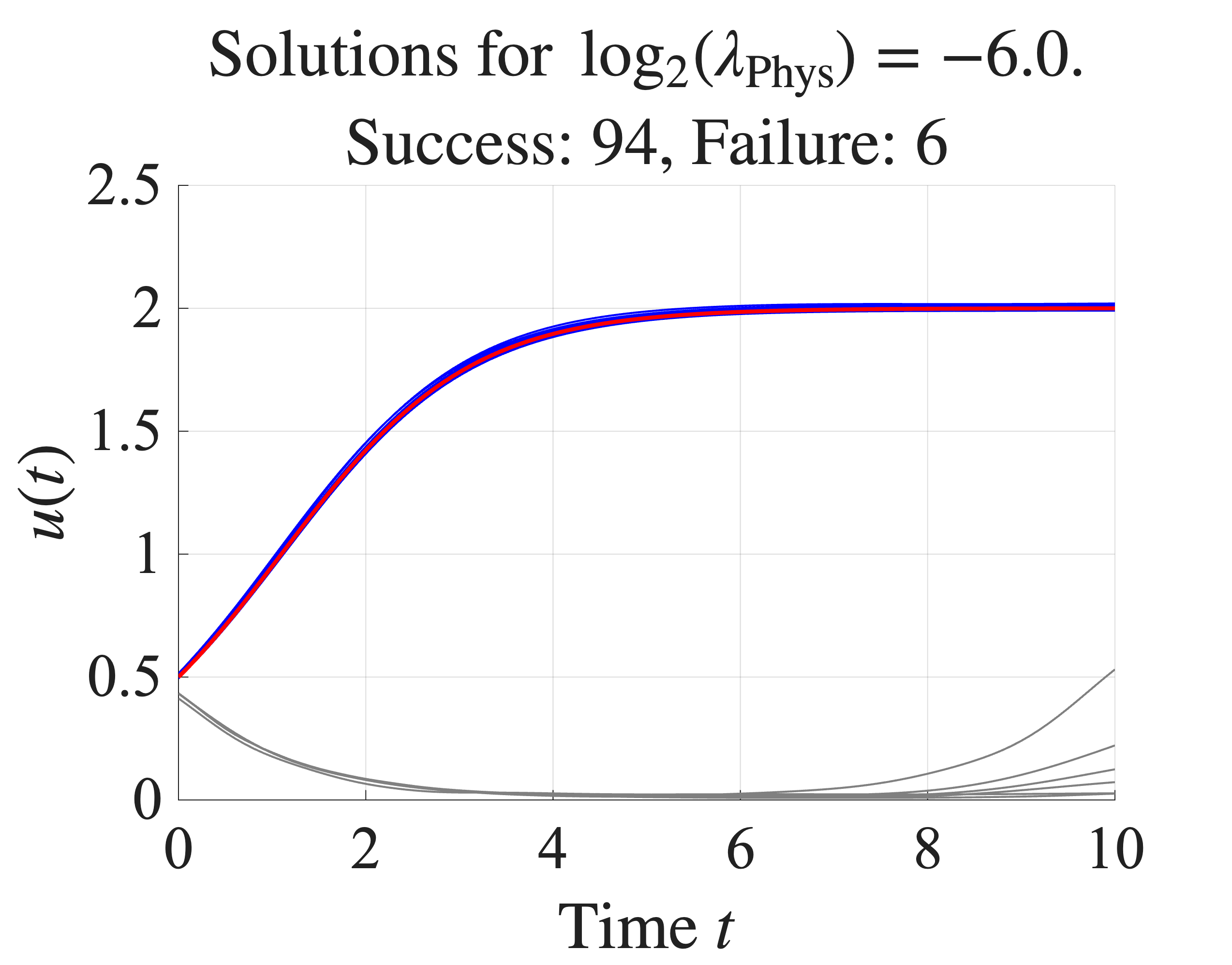} &
        \includegraphics[width=0.24\textwidth]{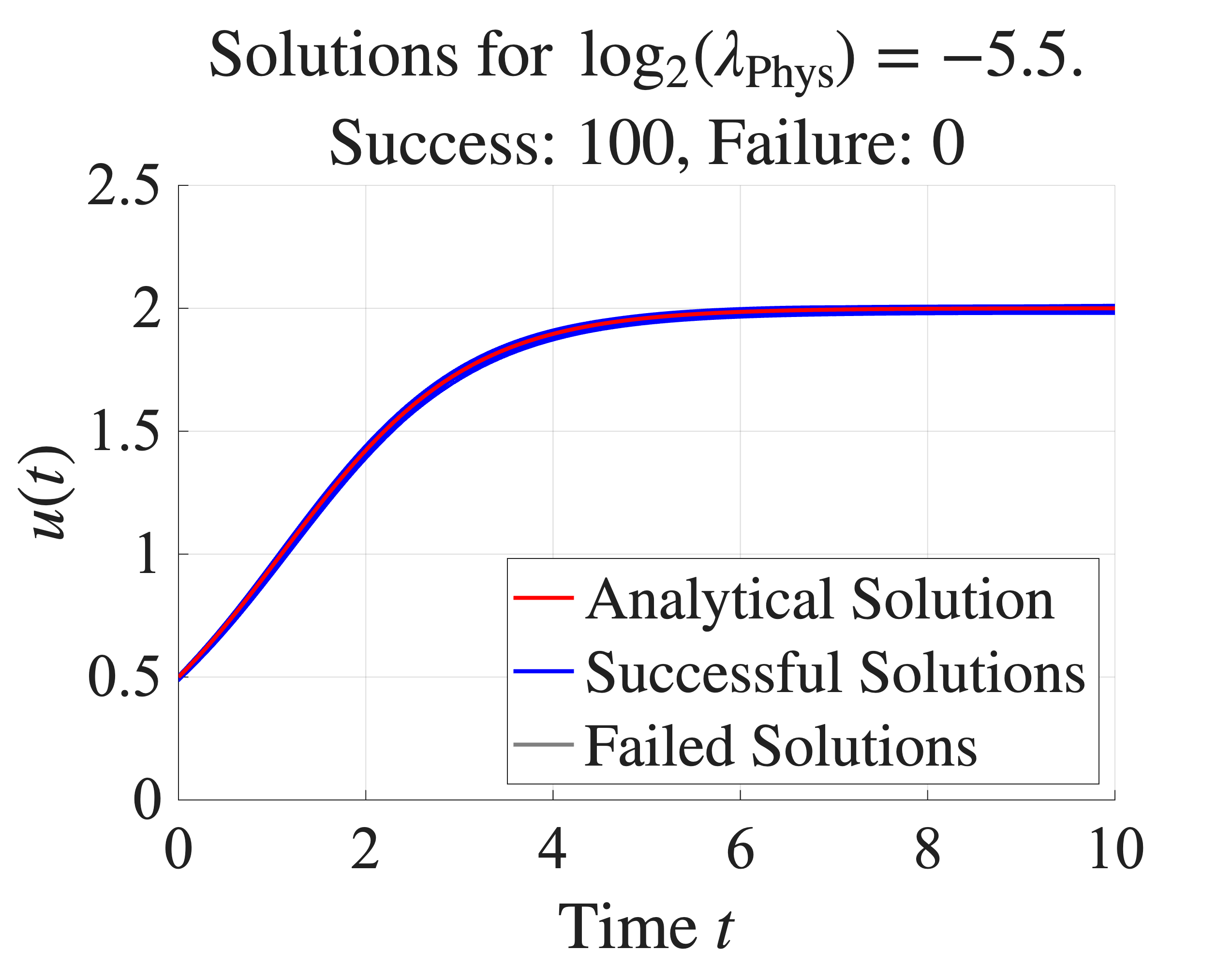} \\
    \end{tabular}
    
    \caption{Numerical solutions for the logistic equation approximated by PINNs under varying regularization strengths. Subplots correspond to values of $\log_2(\lambda_{\text{Phys}})$ ranging from $-8.0$ (top left) to $-5.5$ (bottom right), where $\lambda_{\text{Phys}}$ weights the regularization term $L_{\text{Phys}}$ defined in \eqref{eq:penalty}. Red curves represent the analytical solution (ground truth), blue curves show successful numerical solutions verified by our framework, and gray curves indicate failed solutions.}
    \label{fig:all_solutions}
\end{figure}

For this specific parameter set, the training process frequently failed. We observed that this instability can be effectively circumvented by incorporating the physical regularization term $L_{\text{Phys}}$ defined in \eqref{eq:penalty}. To motivate our approach, we first demonstrate a scenario where a standard PINN fails to capture the true solution. This failure mode critically emphasizes the necessity of rigorous verification. Here, we set $\lambda_{\text{IV}}=1$, employed a learning rate of $0.01$, and trained the model for 100 epochs (with 20 iterations per epoch).

Fig.~\ref{fig:all_solutions} illustrates the training results corresponding to different values of the regularization parameter $\lambda_{\text{Phys}}$, along with the counts of successful and failed verifications. As evident from the figure, under weak regularization regimes, many approximate solutions became trapped in local minima, leading to erroneous predictions. Conversely, strengthening the regularization term significantly improved the success rate. It is crucial to note that while the existence of an analytical solution in this specific example permits direct validation of the PINN's accuracy, blindly relying on PINN outputs in general scenarios, where no such reference exists, carries substantial risk. Our proposed framework addresses this critical gap by validating approximation results without reliance on ground-truth references, making it indispensable for trustworthy scientific machine learning.

The effect of the regularization parameter $\lambda_{\text{Phys}}$ was investigated using the analytic solution as ground truth, and the results are summarized in Figs.~\ref{fig:toy_errors} and \ref{fig:toy_success}. For each value of $\lambda_{\text{Phys}}$, we performed 100 independent trials using different random initializations. A trial was considered successful if the maximum relative error, $\max_{t \in [0,T]} |u_{\hat{\theta}}(t) - u^\star(t)| / u^\star(t)$, was less than or equal to $0.1$ (i.e., $10\%$) after training, where $u^\star$ denotes the analytical solution. Based on these results, we set $\log_2(\lambda_{\text{Phys}}) = -4$ for the subsequent testing.

\begin{figure}[t]
    \centering
    \includegraphics[width=0.7\linewidth]{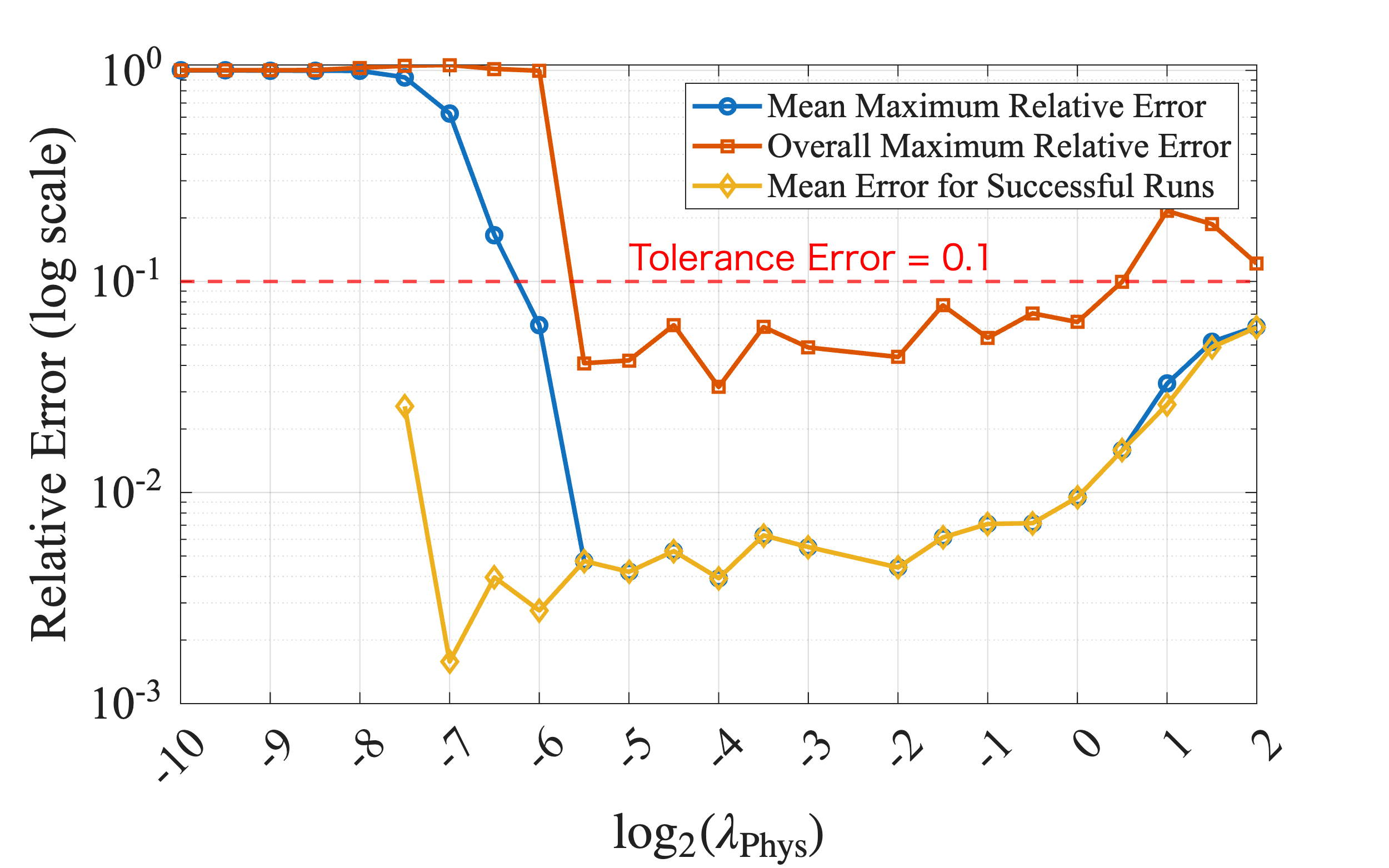}
    \caption{Relative approximation errors for the logistic equation as a function of the regularization parameter $\lambda_{\text{Phys}}$. The horizontal axis represents $\log_2(\lambda_{\text{Phys}})$, while the vertical axis displays the relative error on a logarithmic scale. ``Mean'' denotes the average across 100 trials, whereas ``Overall'' corresponds to the worst-case scenario. The red dashed line indicates the error tolerance threshold of 0.1.}
    \label{fig:toy_errors}
\end{figure}

\begin{figure}[t]
    \centering
    \includegraphics[width=0.7\linewidth]{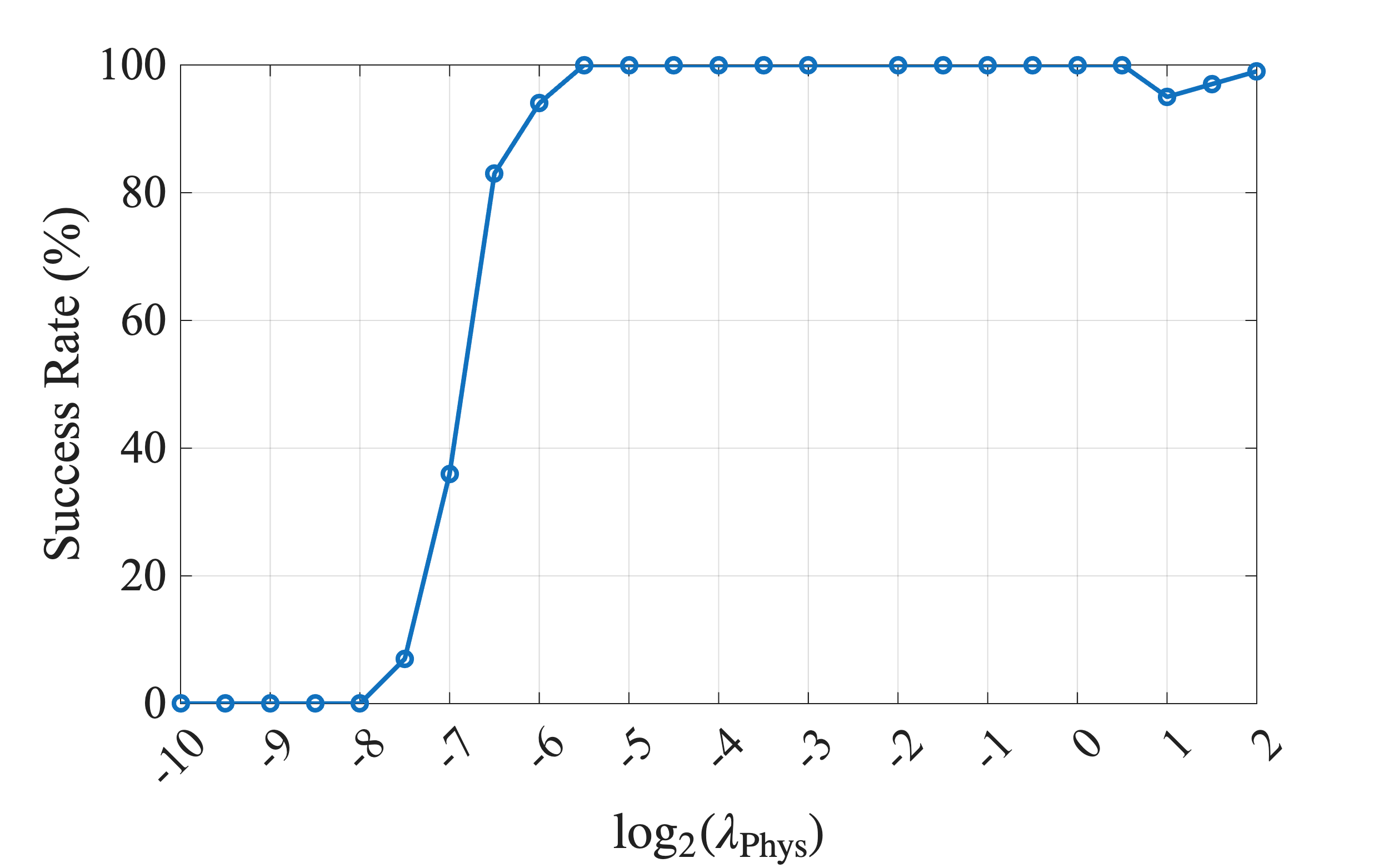}
    \caption{Success rate of approximate solutions for the logistic equation as a function of the regularization parameter $\lambda_{\text{Phys}}$. The success rate demonstrates a sharp, sigmoidal increase within the interval $\log_2(\lambda_{\text{Phys}}) \in [-8, -6]$, eventually plateauing at nearly $100\%$ for $\log_2(\lambda_{\text{Phys}}) \geq -5.5$.}
    \label{fig:toy_success}
\end{figure}

To investigate the characteristics of the proposed Learn and Verify framework, we evaluated the success rate by varying the error tolerance $\varepsilon$ associated with $v_{\ul{\theta}}$ and $w_{\ol{\theta}}$, as well as the number of layers in the neural networks.
We employed a learning rate of $10^{-4}$ and monitored verification success at checkpoints of 100, 200, and 300 epochs. The verification success rates for different combinations of $\varepsilon$ and network depth are summarized in Table~\ref{tab:toy_verification}.
For visual reference, examples of the obtained sub- and super-solutions, along with their corresponding residuals, are presented in Fig.~\ref{fig:toy_enclosure_comparison}.

\begin{table}[t]
\centering
\caption{
  Verification success rates for the sub- and super-solutions of the logistic equation across different error tolerances $\varepsilon$ and network architectures.
}
\label{tab:toy_verification}
\begin{tabular}{c c c c c}
\hline
{Tolerance} & 
\multirow{2}{*}{{\#Layers}} & 
\multicolumn{3}{c}{{Success rate (\%)}}\\
\cline{3-5}
{param} ${\varepsilon}$ & & 100th epoch & 200th epoch & 300th epoch\\
\hline
\multirow{4}{*}{$2^{-4}$}
 & 3 & 70 & 91 & 98 \\
 & 4 & 87 & 95 & 98 \\
 & 5 & 82 & 95 & 100 \\
 & 6 & 90 & 99 & 100 \\
\hline
\multirow{4}{*}{$2^{-5}$}
 & 3 & 45 & 92 & 99 \\
 & 4 & 74 & 98 & 100 \\
 & 5 & 84 & 100 & 100 \\
 & 6 & 79 & 98 & 100 \\
\hline
\multirow{4}{*}{$2^{-6}$}
 & 3 & 14 & 86 & 98 \\
 & 4 & 59 & 96 & 100 \\
 & 5 & 65 & 94 & 98 \\
 & 6 & 55 & 97 & 99 \\
\hline
\multirow{4}{*}{$2^{-7}$}
 & 3 & 0 & 1 & 13 \\
 & 4 & 0 & 3 & 25 \\
 & 5 & 0 & 4 & 18 \\
 & 6 & 0 & 4 & 13 \\
\hline
\rule[-0.6em]{0pt}{1.8em}$2^{-8}$ & 3--6 & 0 & 0 & 0 \\
\hline
\end{tabular}
\end{table}

\begin{figure}[t]
    \centering
    \begin{tabular}{@{}cc@{}}
        \includegraphics[width=0.48\linewidth]{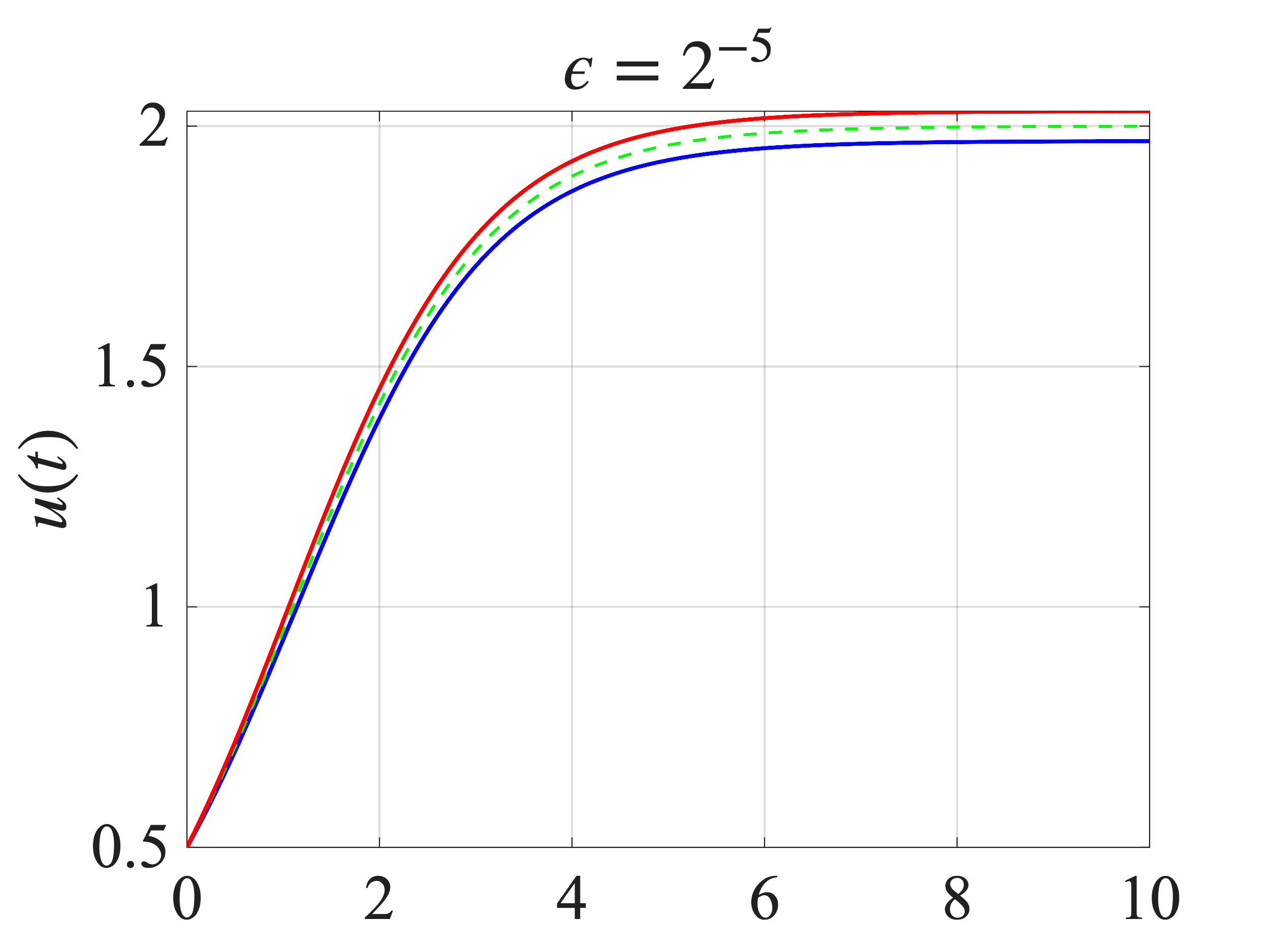} &
        \includegraphics[width=0.48\linewidth]{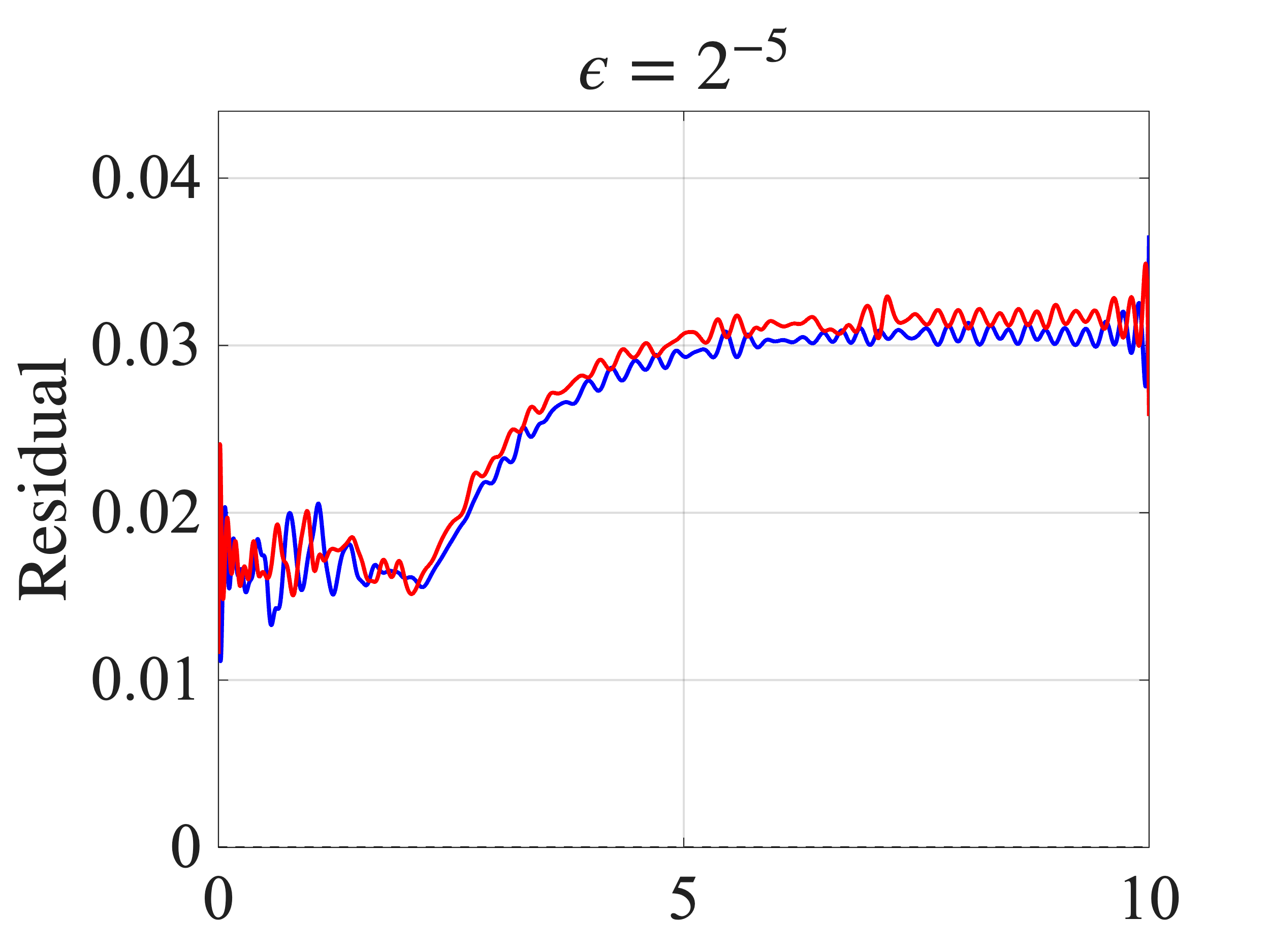} \\[0.5em]
        \includegraphics[width=0.48\linewidth]{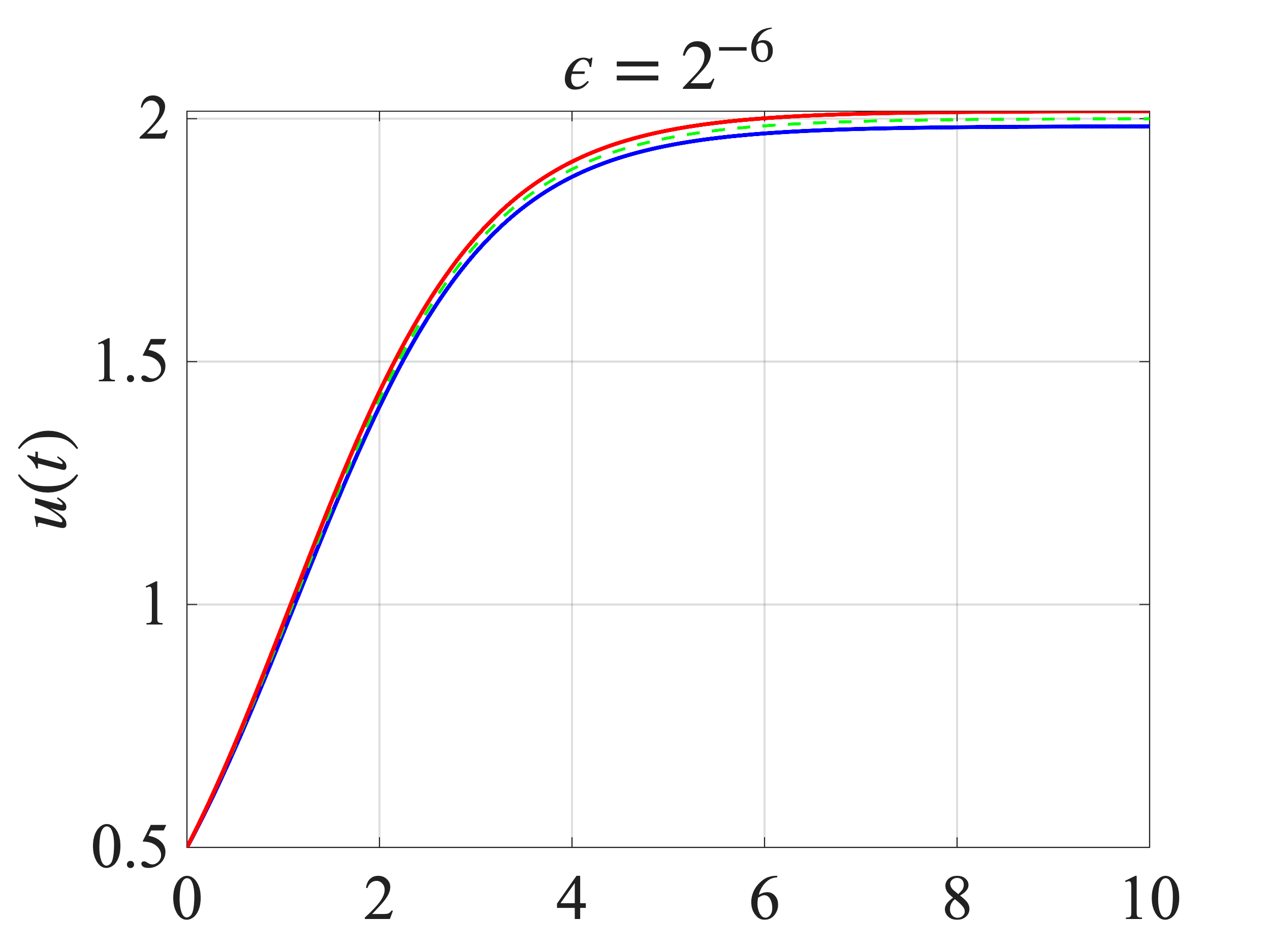} &
        \includegraphics[width=0.48\linewidth]{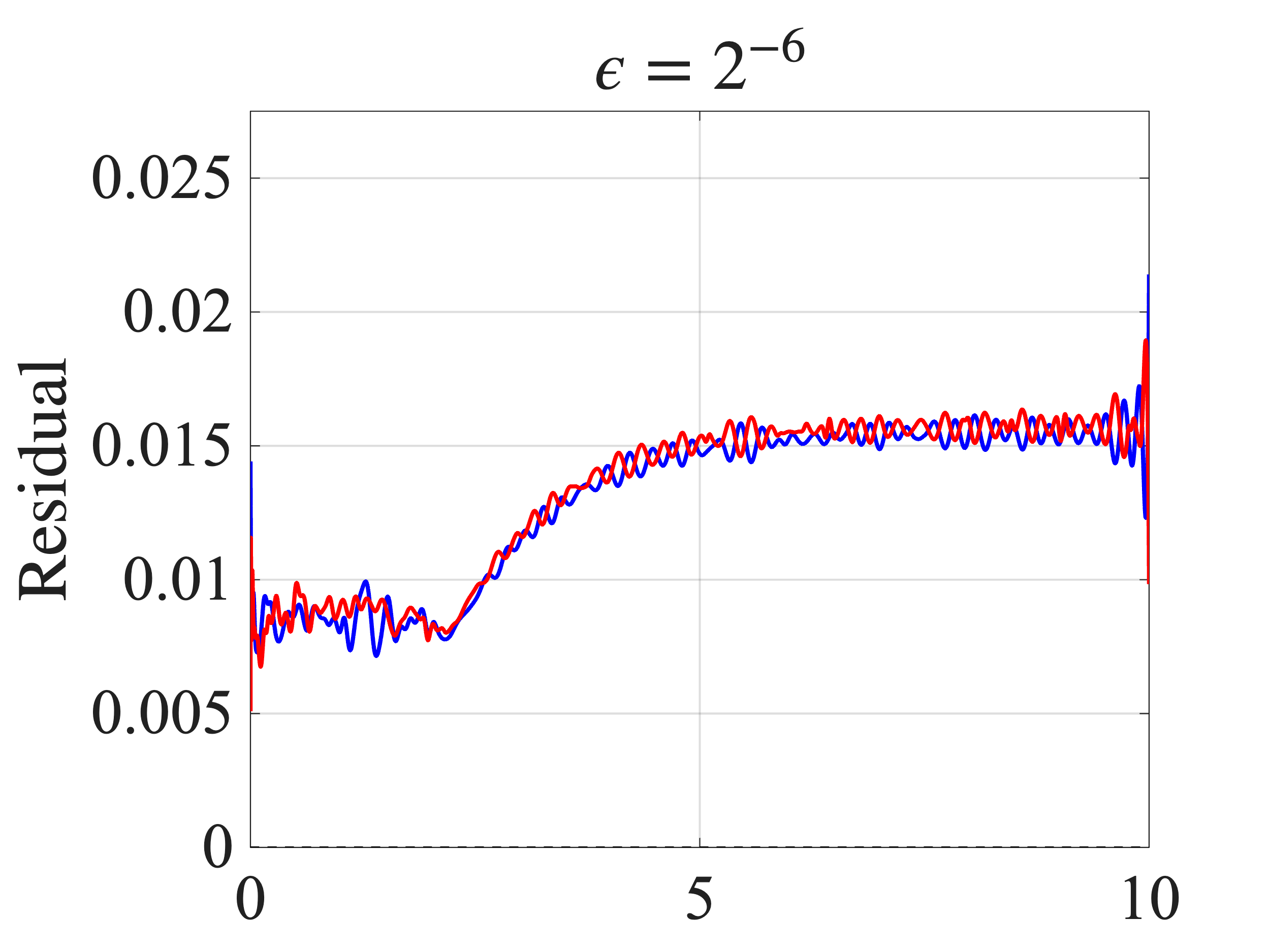} \\[0.5em]
        \includegraphics[width=0.48\linewidth]{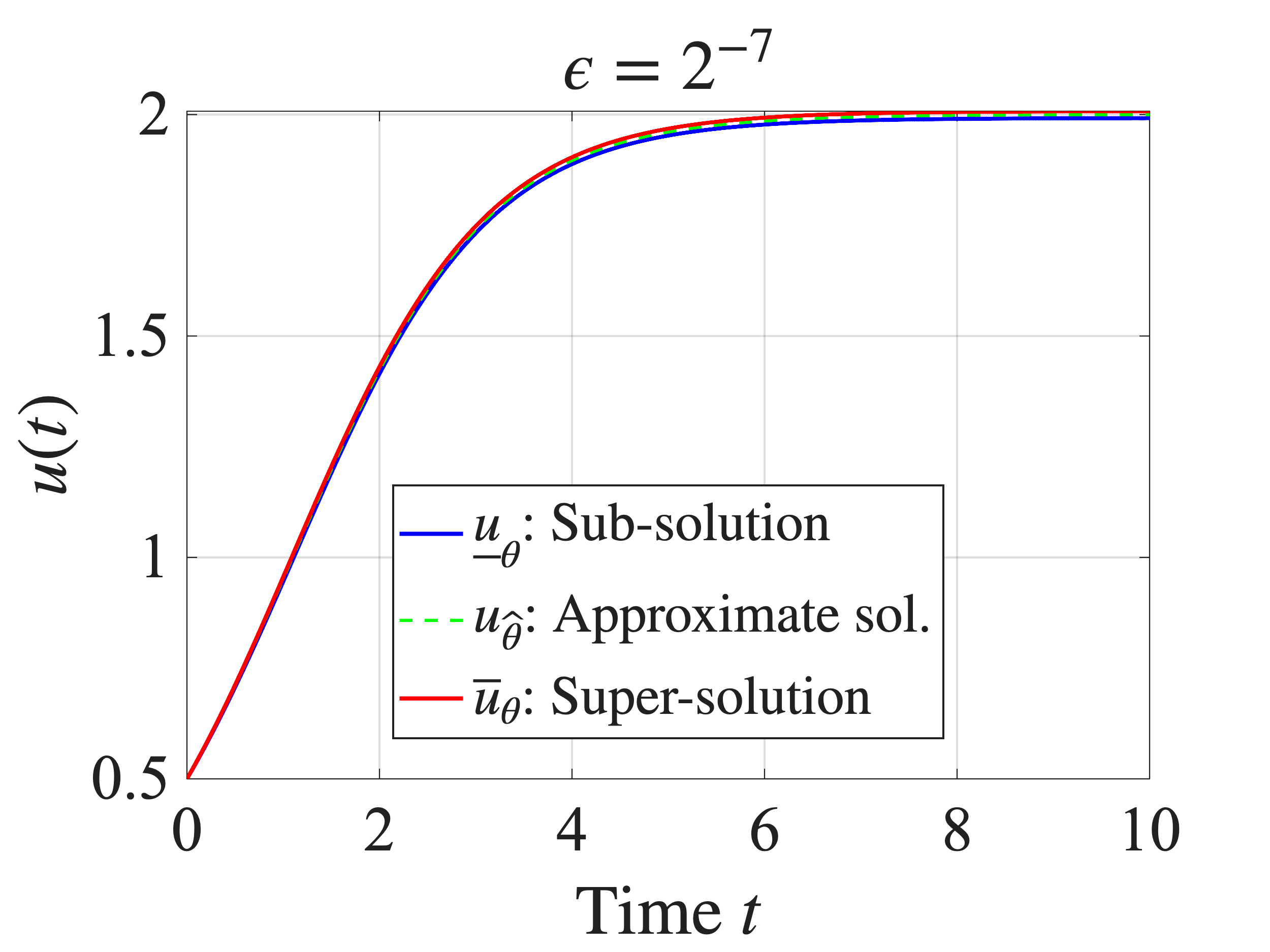} &
        \includegraphics[width=0.48\linewidth]{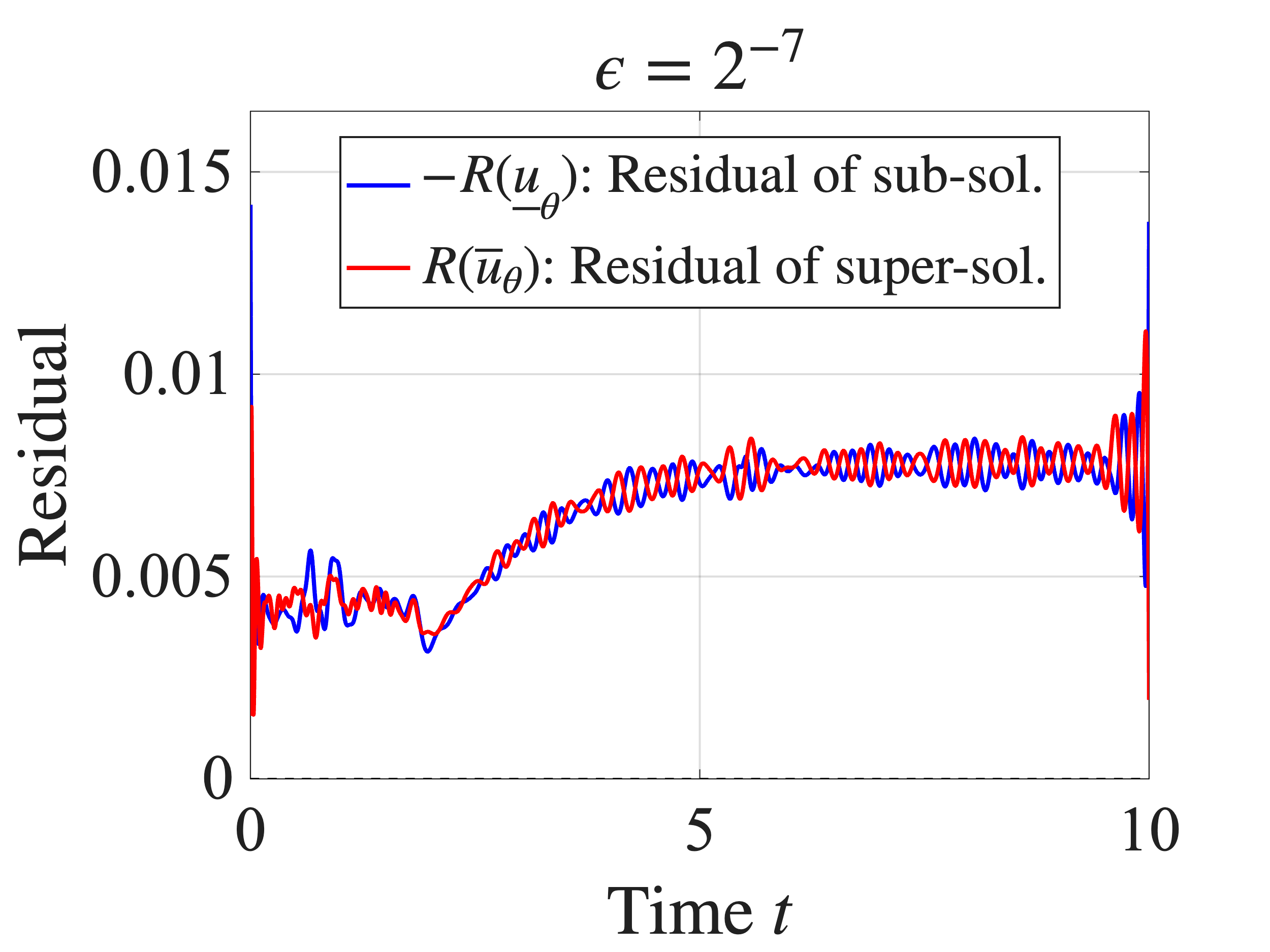}
    \end{tabular}
    \caption{Solution enclosures (left) and ODE residuals (right) for three error tolerances: $\varepsilon = 2^{-5}, 2^{-6}, 2^{-7}$ (from top to bottom). The sign of the residuals is chosen such that non-negativity correspond to a successful verification.}
    \label{fig:toy_enclosure_comparison}
\end{figure}

These results yield some insights regarding the verification of sub- and super-solutions. First, a tighter error tolerance increases the difficulty of training, as evidenced by the increase in the number of required epochs. This can be intuitively understood from Fig.~\ref{fig:toy_enclosure_comparison}, as a smaller $\varepsilon$ imposes stricter constraints, requiring the residual to approach zero without crossing it. Second, the inherent error of the approximate solution $u_{\hat{\theta}}$ can be inferred from the boundary between verification success and failure. In this instance, the sharp transition suggests that the approximation error lies between $2^{-7}$ and $2^{-8}$. Although this range might be an overestimate due to the conservative nature of interval arithmetic, the mean error shown in Fig.~\ref{fig:toy_errors} is consistent with the estimation provided by the proposed framework.

\subsection{Generalized Logistic Equation}

Next, we consider the generalized logistic equation, which extends the classical model by incorporating time-dependent coefficients. The dynamics are governed by 
\begin{align}
	\label{eq:logistic-gene}
	\left\{\begin{array}{l l}
		\displaystyle\frac{du}{dt}=f(t,u):=r(t)u\left(1-\frac{u}{k(t)}\right), &t \in (0,T),\\
		u(0)=a,\\
	\end{array}\right.
\end{align}
where $u=u(t)$, and the coefficients $r(t)$ and $k(t)$ represent the time-varying intrinsic growth rate and carrying capacity, respectively. For this experiment, we define them as
\begin{align*}
r(t)=r_0 (1+\sin(\alpha t)), \;\quad\;
k(t)=k_0(\log(1+t)+1).
\end{align*}
In our numerical experiments, we utilized the following parameter set: $T=10$, $r_0=k_0=2$, and $\alpha=10$.

Unlike the classical case, the explicit time dependence of $r$ and $k$ generally precludes the derivation of a closed-form analytical solution. While standard numerical integrators (e.g., Runge-Kutta methods) can provide efficient approximate trajectories, they typically lack the ability to certify the accuracy of the result. In contrast, our proposed framework constructs a rigorous functional enclosure, providing mathematically guaranteed error bounds even in the absence of an exact reference solution. This effectively guarantees the location of the true solution, which would otherwise remain uncertain.

To illustrate the increased difficulty of this problem compared to the classical logistic equation discussed in the previous subsection, we generated plots analogous to Figs.~\ref{fig:toy_errors} and \ref{fig:toy_success}, employing a high-precision numerical solution as the reference. The results are summarized in Figs.~\ref{fig:general_errors} and \ref{fig:general_success}. Although the high error and low success rates suggest that the problem is challenging, these metrics are inherently unreliable due to the absence of a closed-form solution.
Based on these results, we set $\log_2(\lambda_{\text{Phys}}) = 8$ for the subsequent testing.

\begin{figure}[t]
    \centering
    \includegraphics[width=0.7\linewidth]{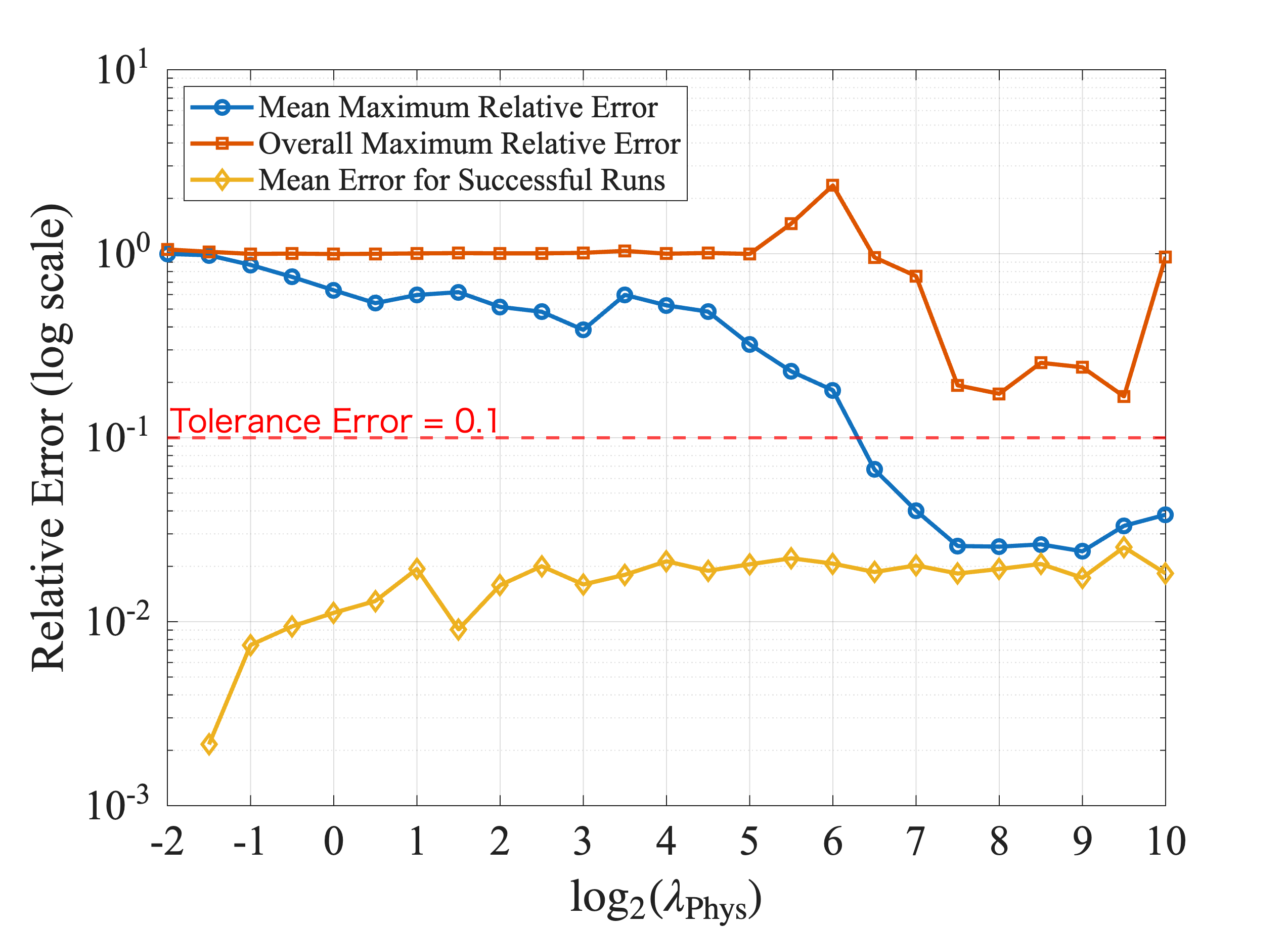}
    \caption{Relative approximation errors for the generalized logistic equation as a function of the regularization parameter $\lambda_{\text{Phys}}$. In contrast to Fig.~\ref{fig:toy_errors}), this figure was generated using a reference numerical solution obtained via MATLAB’s ODE45 solver with stringent tolerances ($\text{RelTol} = \text{AbsTol} = 10^{-12}$, $\text{MaxStep} = 0.01$) over a fine grid of 10000 time points.}
    \label{fig:general_errors}
\end{figure}

\begin{figure}[t]
    \centering
    \includegraphics[width=0.7\linewidth]{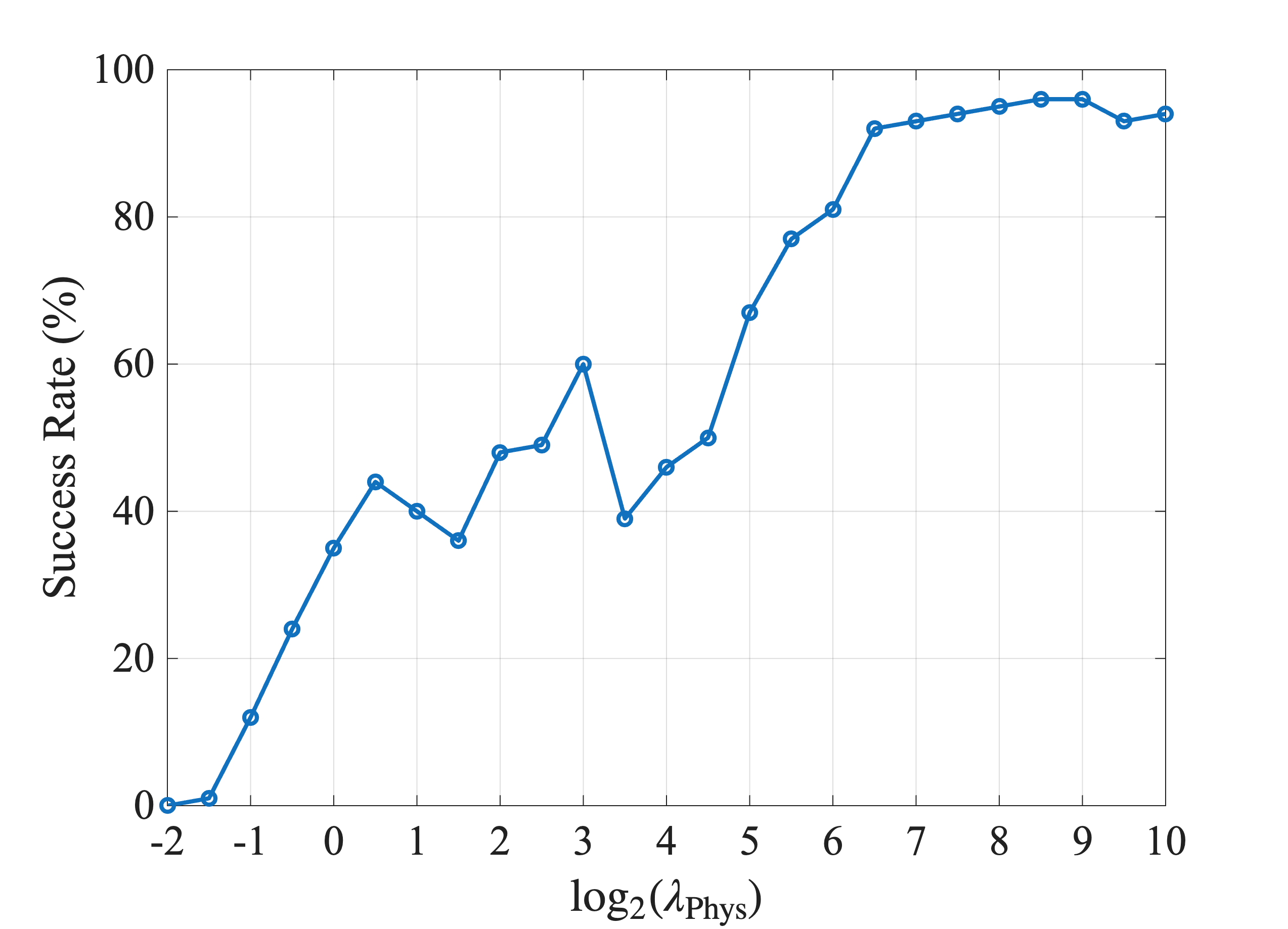}
    \caption{Success rate of approximate solutions for the generalized logistic equation as a function of the regularization parameter $\lambda_{\text{Phys}}$ (see Fig.~\ref{fig:general_errors}).}
    \label{fig:general_success}
\end{figure}

\begin{table}[t]
\centering
\caption{
  Verification success rates for the sub- and super-solutions of the generalized logistic equation across different error tolerances $\varepsilon$ and network architectures.
}
\label{tab:general_verification}
\begin{tabular}{c c c c c c c}
\hline
{Tolerance} & 
\multirow{2}{*}{{\#Layers}} & 
\multicolumn{5}{c}{{Success rate (\%)}}\\
\cline{3-7}
{param} ${\varepsilon}$& & 100th & 200th & 300th & 400th & 500th\\
\hline
\multirow{5}{*}{$2^{-4}$}
 & 2 & 0 & 1 & 32 & 72 & 92 \\
 & 3 & 3 & 73 & 89 & 97 & 99 \\
 & 4 & 20 & 86 & 97 & 97 & 100 \\
 & 5 & 26 & 84 & 95 & 99 & 99 \\
 & 6 & 24 & 81 & 97 & 99 & 99 \\
\hline
\multirow{5}{*}{$2^{-5}$}
 & 2 & 0 & 0 & 0 & 0 & 0 \\
 & 3 & 0 & 5 & 83 & 100 & 100 \\
 & 4 & 0 & 29 & 93 & 97 & 100 \\
 & 5 & 0 & 37 & 94 & 100 & 100 \\
 & 6 & 0 & 31 & 82 & 99 & 100 \\
\hline
\multirow{5}{*}{$2^{-6}$}
 & 2 & 0 & 0 & 0 & 0 & 0 \\
 & 3 & 0 & 0 & 2 & 79 & 97 \\
 & 4 & 0 & 0 & 13 & 84 & 99 \\
 & 5 & 0 & 0 & 5 & 35 & 82 \\
 & 6 & 0 & 0 & 1 & 11 & 41 \\
\hline
\end{tabular}
\end{table}

\begin{figure}[t]
    \centering
    \begin{tabular}{@{}cc@{}}
        \includegraphics[width=0.48\linewidth]{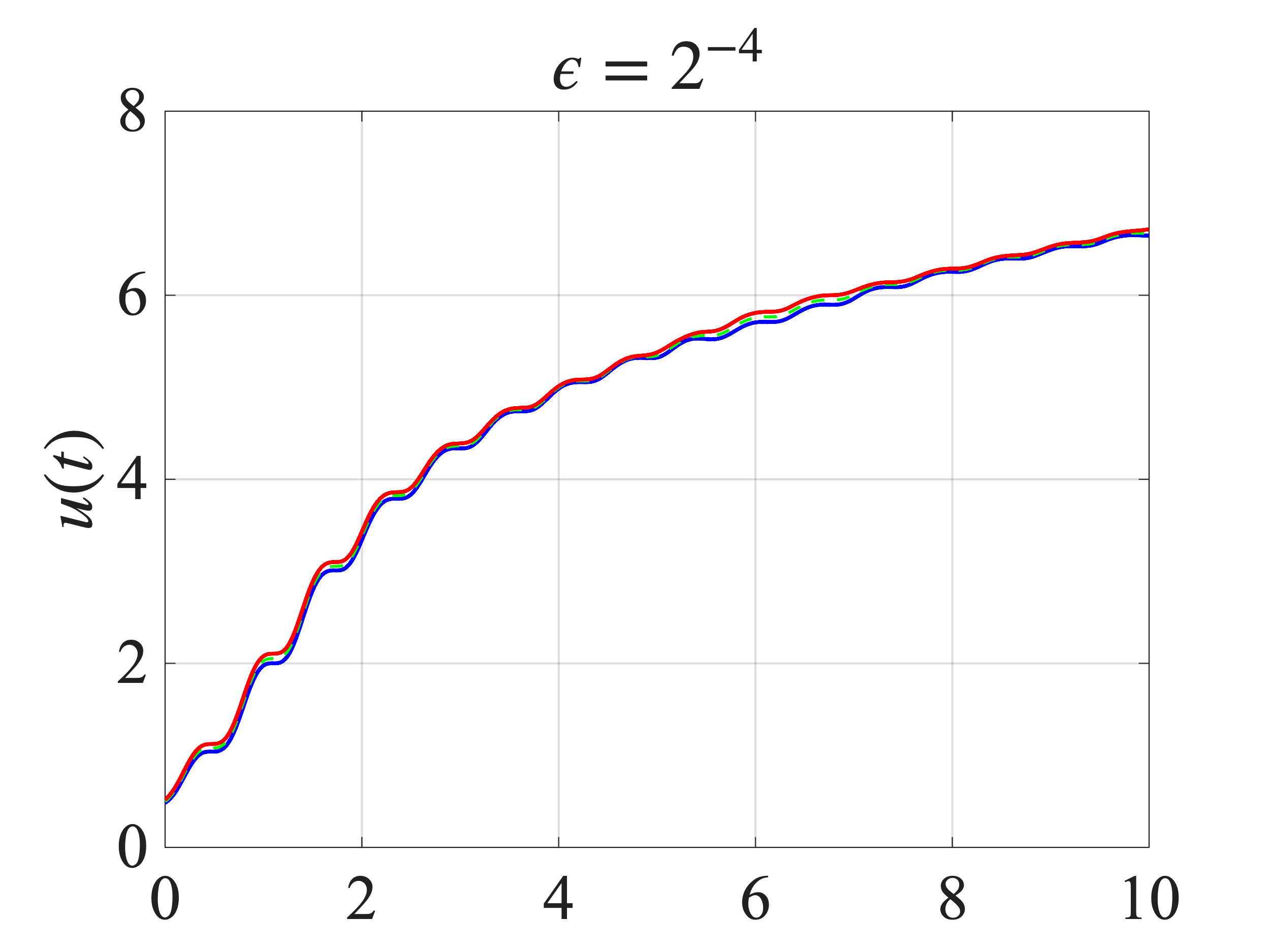} &
        \includegraphics[width=0.48\linewidth]{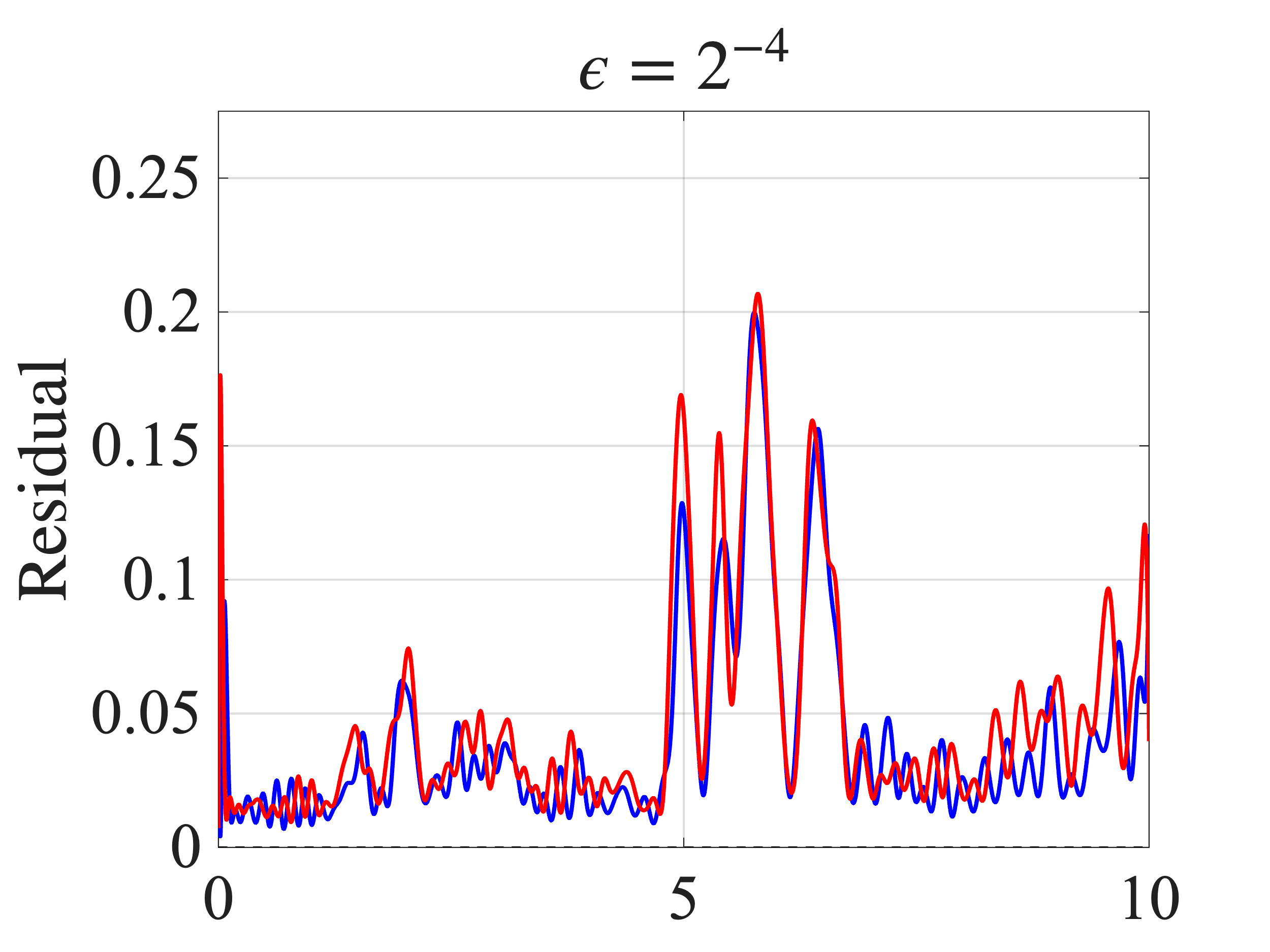} \\[0.5em]
        \includegraphics[width=0.48\linewidth]{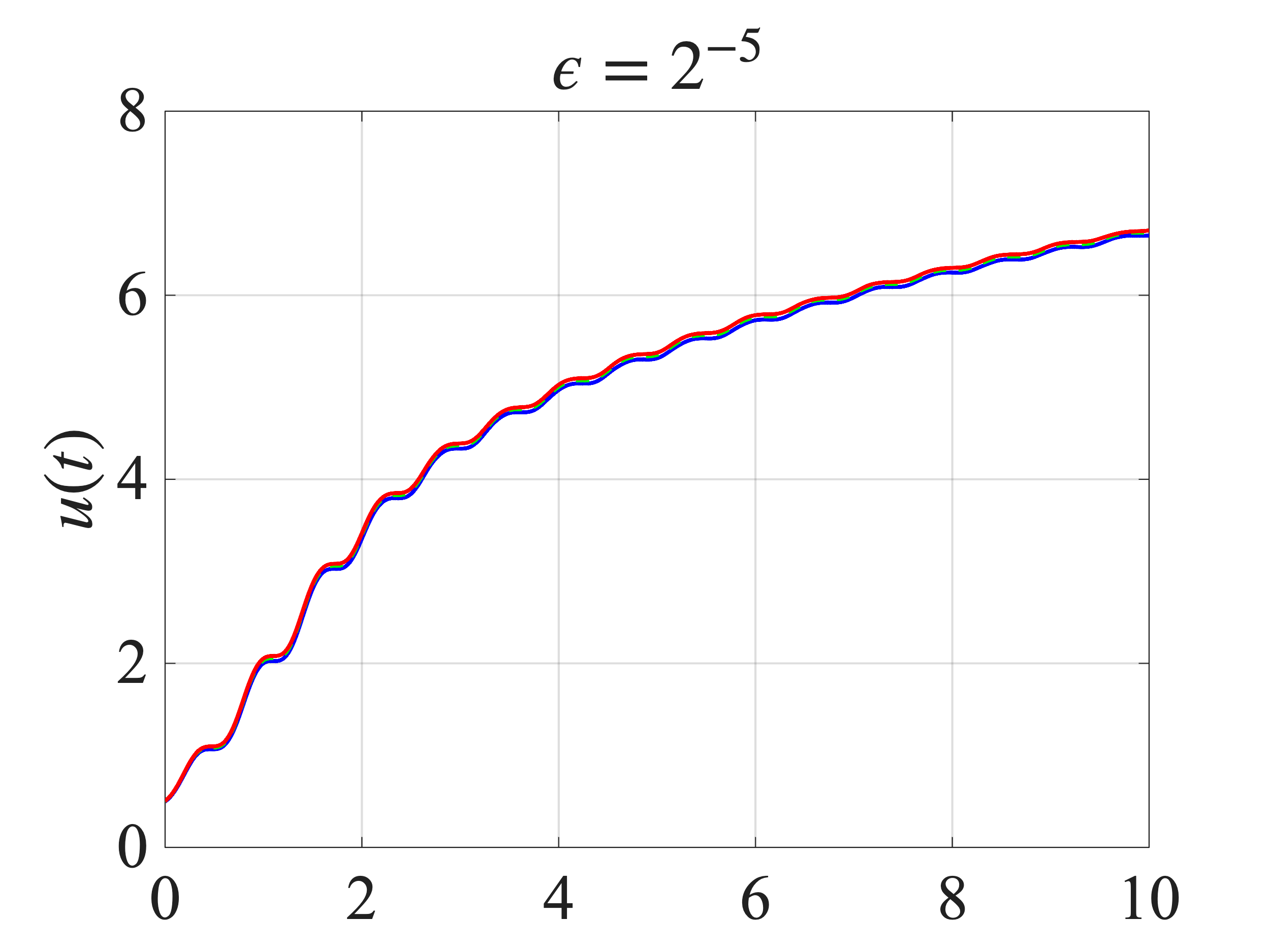} &
        \includegraphics[width=0.48\linewidth]{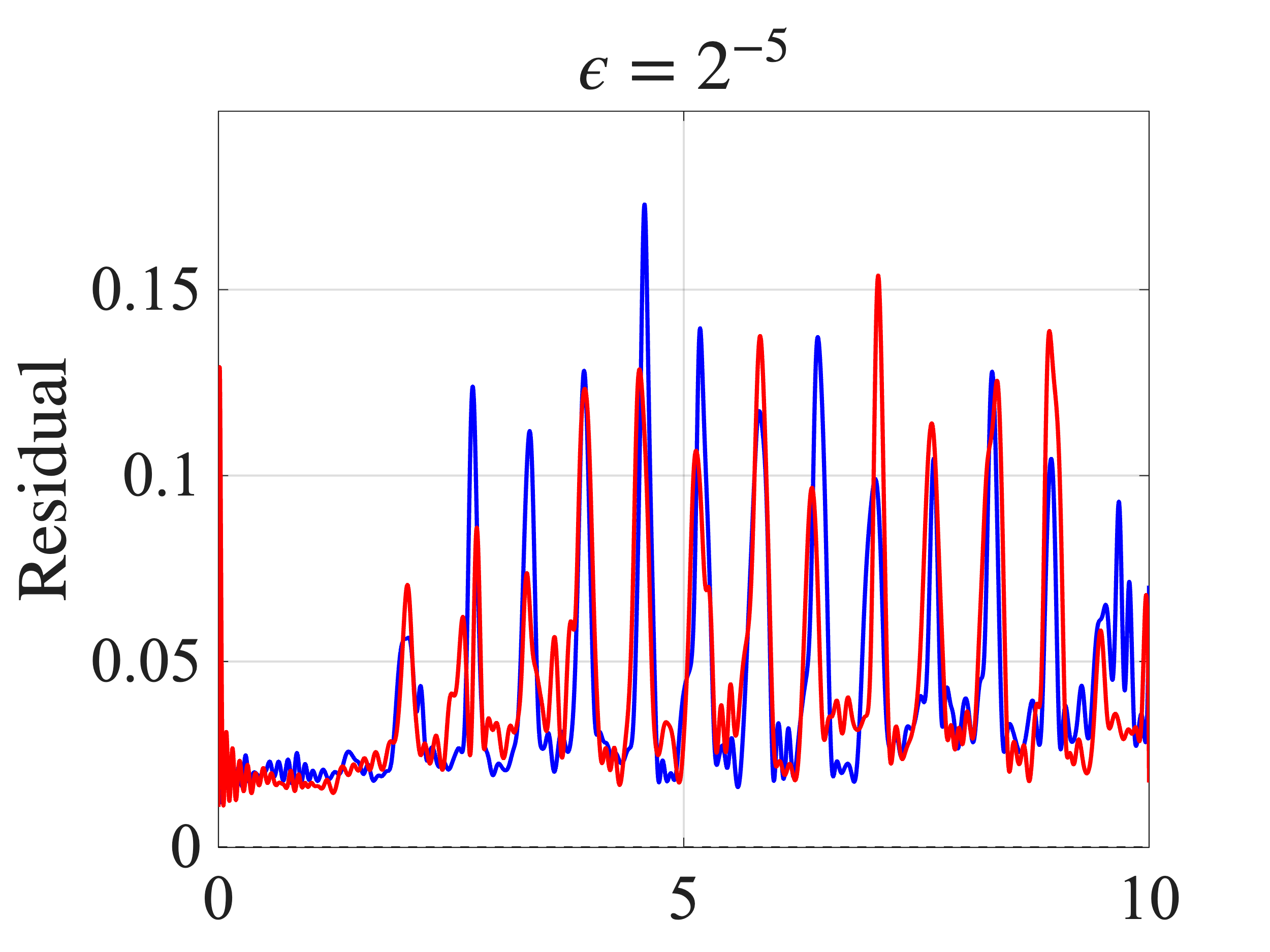} \\[0.5em]
        \includegraphics[width=0.48\linewidth]{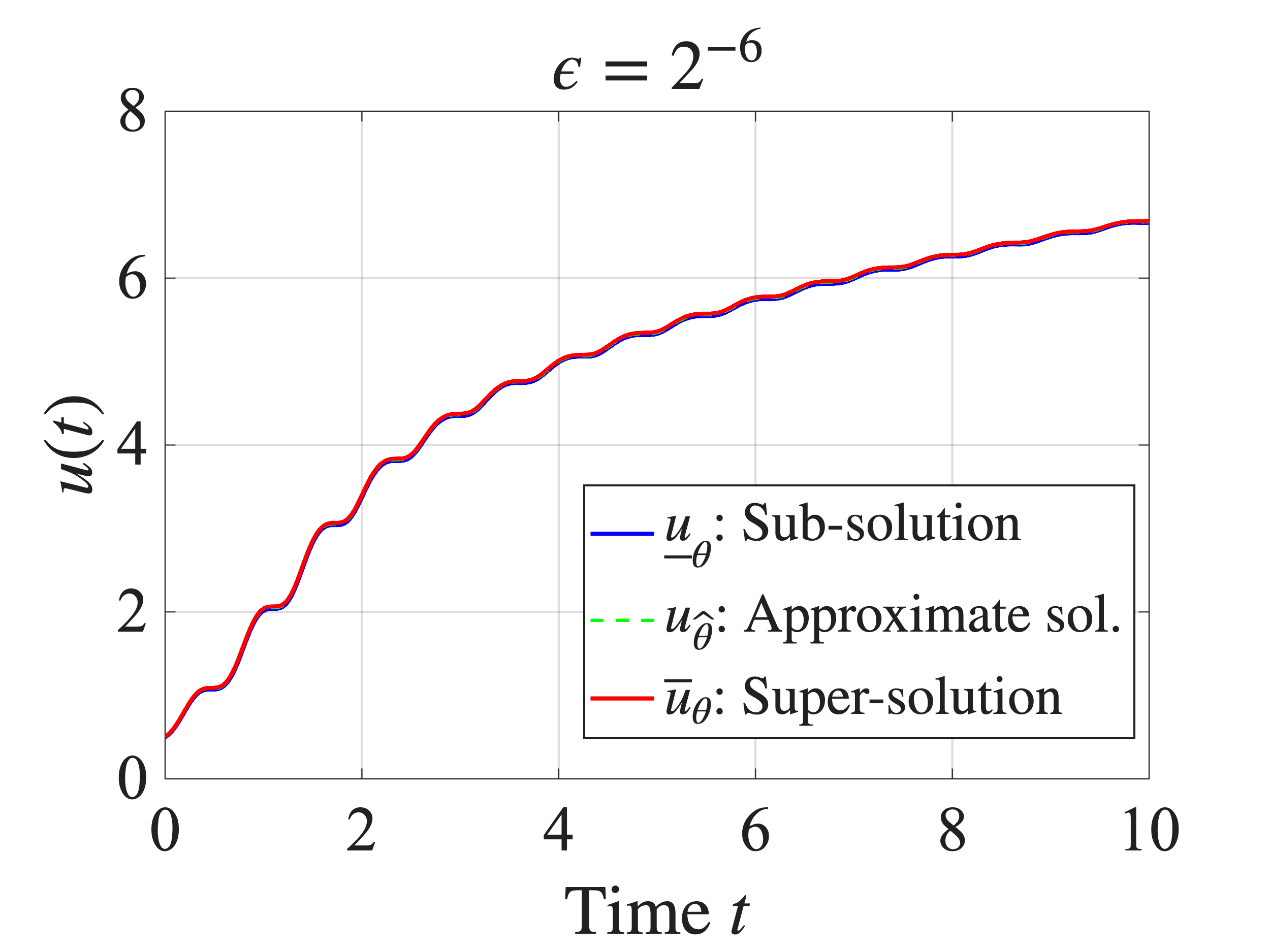} &
        \includegraphics[width=0.48\linewidth]{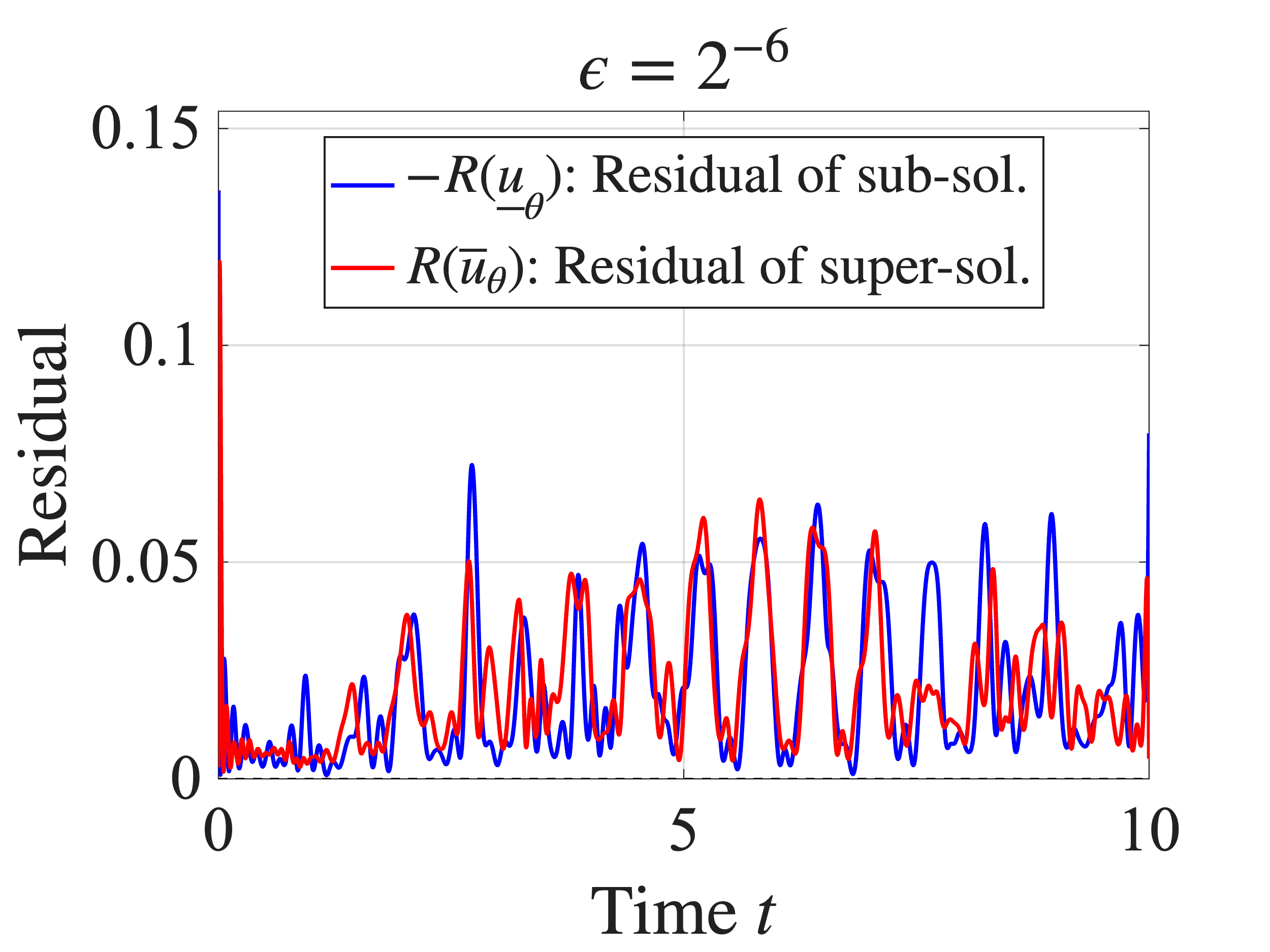}
    \end{tabular}
    \caption{Solution enclosures (left) and ODE residuals (right) for three error tolerances: $\varepsilon = 2^{-4}, 2^{-5}, 2^{-6}$ (from top to bottom). The sign of the residuals is chosen such that non-negativity correspond to a successful verification.}
    \label{fig:general_enclosure_comparison}
\end{figure}

The verification success rates are summarized in Table~\ref{tab:general_verification}, and examples of the obtained sub- and super-solutions, along with their corresponding residuals, are presented in Fig.~\ref{fig:general_enclosure_comparison}. These results yield several key observations. First, sufficient representational capacity is required to approximate the solution; notably, networks with only 2 layers (i.e., 1 hidden layer) were unsuccessful in validation for tolerances $\varepsilon\leq2^{-5}$. Second, training is more difficult compared to the previous case, as evidenced by the increased number of required epochs. This difficulty is further highlighted by the fact that the 4-layer network outperformed the 6-layer network, likely due to a superior balance between representational capacity and trainability. Conventionally, users must heuristically tune this trade-off without a reliable reference. In contrast, the proposed Learn and Verify framework enables parameter selection based on rigorous error bounds, thereby facilitating automatic tuning.

\subsection{Blow-up Solution to Riccati Equation}

Finally, we demonstrate the capability of the proposed framework to handle finite-time blow-up, a scenario where the solution trajectory diverges to infinity within a finite horizon.
Such blow-up phenomena are critical in various physical applications, including reaction-diffusion systems, nonlinear wave equations, and geometric flows, where the blow-up time often corresponds to a phase transition or system failure.
A distinguishing feature of our approach is its ability to rigorously capture the blow-up time by constructing sub- and super-solutions that enclose the singularity.
To illustrate this capability, we consider the Riccati equation: 
\begin{align}
	\label{eq:riccati}
	\frac{du}{dt}(t)=t^2+u(t)^2.
\end{align}
It is well-established that solutions to this equation originating from any non-negative initial value $u(0) = a \ge 0$ will exhibit finite-time blow-up. For simplicity, we focus on the specific case where $u(0)=0$.
In this scenario, the solution diverges near $t=2$ (specifically, slightly beyond $t=2$), as illustrated in Fig.~\ref{fig:Riccati_blowup}. The rapid growth of the solution as $t$ approaches the blow-up time exemplifies the challenges associated with numerical integration near a singularity.

\begin{figure}[t]
    \centering
    \includegraphics[width=0.6\linewidth]{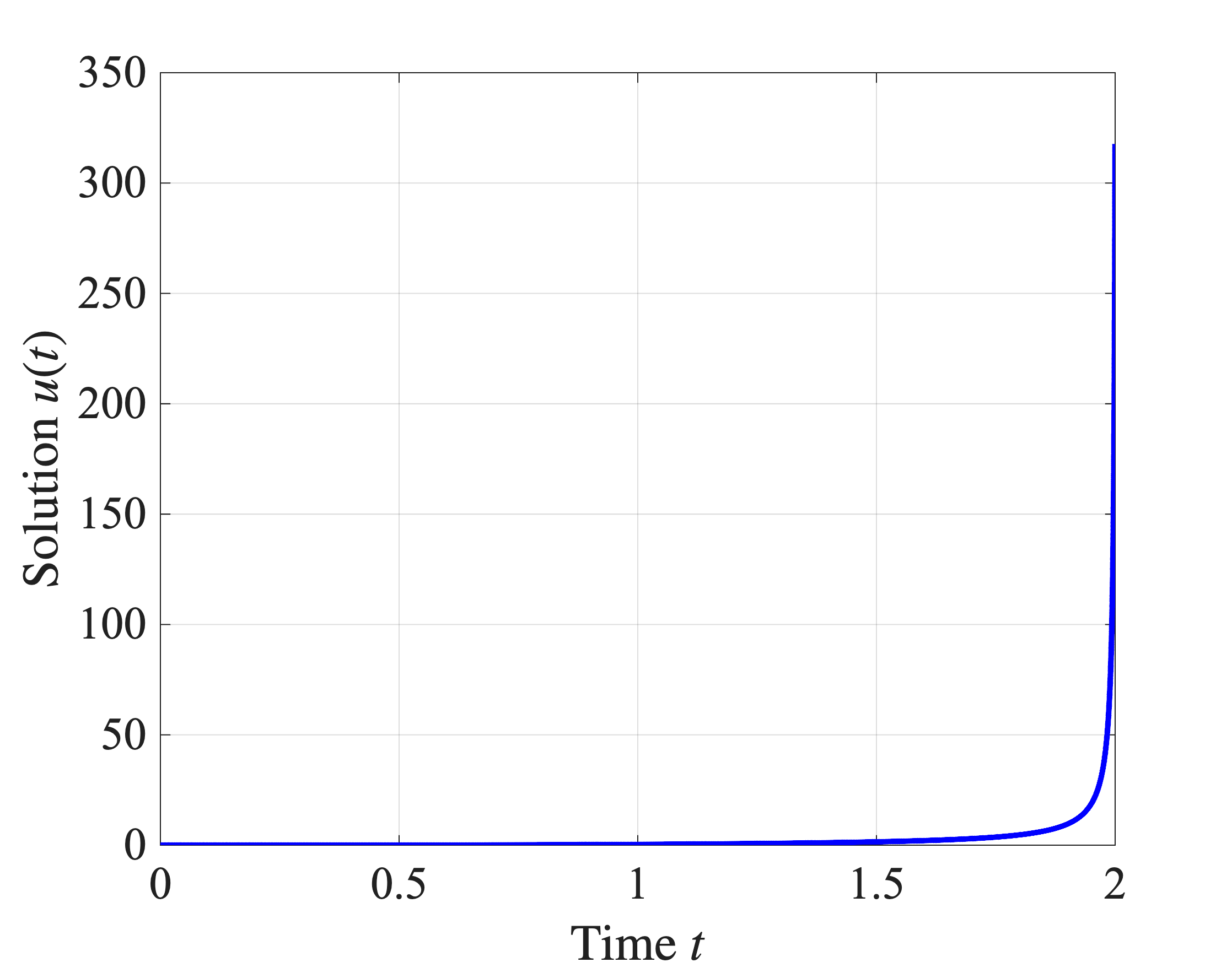}
    \caption{Finite-time blow-up of the solution to the Riccati equation with initial condition $u(0) = 0$. The solution approaches a vertical asymptote at the blow-up time $T \approx 2$.}
    \label{fig:Riccati_blowup}
\end{figure}

To address the finite-time blow-up, we propose a four-step strategy: (1) apply a variable transformation to regularize the singularity; (2) construct sub- and super-solutions on a finite interval up to a predetermined lower bound of the blow-up time; (3) verify the sub- and super-solutions to guarantee the existence of the true solution within the computational domain; and (4) extend the sub-solution using an analytic lower bound to derive an upper bound for the blow-up time. As this approach necessitates a problem-specific formulation, we detail its application to the Riccati equation.

The first step is the variable transformation. A solution exhibiting finite-time blow-up typically possesses a singularity of the form $u(t) \sim (T_{\text{bu}}-t)^{-\alpha}$, where $T_{\text{bu}}$ represents the blow-up time, and $\alpha$ denotes the blow-up rate. Since the generic nonlinearity $du/dt = u^p$ implies a divergence scaling of $u \sim (T_{\text{bu}}-t)^{-1/(p-1)}$, the quadratic term in the Riccati equation ($p=2$) yields $\alpha=1$. Based on this asymptotic behavior, we employ the following transformation:
\begin{align}
    u(t) = \frac{1}{c-t}\phi(t)\quad(c>0),
\end{align}
where $c$ is an estimate of the blow-up time. This regularization strategy is applicable to general equations, provided the blow-up rate can be estimated.

Applying the variable transformation to the original ODE yields the following equation for $\phi$:
\begin{align}
    \frac{d\phi}{dt}(t) = (c-t)\,t^2 +\frac{\phi(t)^2 - \phi(t)}{c-t}.
\end{align}
Subsequently, we approximate $\phi$ and construct sub- and super-solutions on the interval $t \in [0, \tilde{T}]$, where $\tilde{T}$ is a predetermined lower bound for the blow-up time. In this instance, we selected $\tilde{T} = 2 - 2^{-4}$ ($= 1.9375$). Because the computational domain $[0, \tilde{T}]$ excludes the singularity, the neural networks can be trained using standard procedures (provided $\tilde{T}<c$).

The computed sub- and super-solutions are then subjected to the verification procedure defined in our framework. A successful verification guarantees the existence of the true solution within the derived enclosure. Crucially, this implies that the solution remains finite on the interval $[0, \tilde{T}]$, thereby establishing $\tilde{T}$ as a rigorous lower bound for the true blow-up time. Although the parameters $c$ and $\tilde{T}$ must be selected by users, they can be iteratively refined based on the verification outcomes%
\footnote{Both of these parameters could potentially be treated as trainable parameters, optimized simultaneously with the sub- and super-solutions to achieve tighter enclosures. This exploration is left for future investigation.}.
Such systematic adjustment is significantly more efficient than groundless heuristics, as the rigorous error bounds provide reliable feedback for parameter tuning.

To determine an upper bound for the true blow-up time $T_{\text{bu}}$, we extend the sub-solution beyond the computational domain. Starting from $t = \tilde{T}$, we construct a sub-solution for the interval $\tilde{T} \leq t \leq T_{\text{bu}}$ by considering an auxiliary ODE that admits an analytic solution and strictly lower-bounds the true solution%
\footnote{We remark that the availability of an explicit analytic solution is not a strict requirement. Any method capable of establishing a rigorous upper bound for the true blow-up time $T_{\text{bu}}$ is sufficient. We selected this specific approach primarily to simplify the exposition.}.
Specifically, we consider
\begin{align}
    \label{eq:ODEforSubSol}
    \frac{d\ul{u}}{dt}(t) = \ul{u}(t)^2
\end{align}
with the initial condition $\ul{u}(\tilde{T})$ taken from the numerically computed sub-solution $\ul{u}_\theta(\tilde{T})$. Since the right-hand side of \eqref{eq:riccati} strictly dominates that of \eqref{eq:ODEforSubSol} (i.e., $t^2 + u^2 \ge u^2$), the solution to \eqref{eq:ODEforSubSol} is, by definition, a sub-solution to \eqref{eq:riccati} and thus remains below the true solution. Analytic solution to \eqref{eq:ODEforSubSol} is
\begin{align}
    \ul{u}(t) = \frac{\ul{u}_\theta(\tilde{T})}{1-\ul{u}_\theta(\tilde{T})\,t} \quad(\tilde{T} \leq t \leq T_{\text{bu}}),
\end{align}
which diverges at $t = 1/\ul{u}_\theta(\tilde{T})$. Since the true solution $u(t)$ is greater than $\ul{u}$, it must blow up no later than $t = 1/\ul{u}_\theta(\tilde{T})$. Thus, we obtain the rigorous upper bound:
\begin{align}
    T_{\text{bu}}\leq\tilde{T}+\frac{1}{\ul{u}_\theta(\tilde{T})}.
\end{align}
Evidently, a tighter (larger) value for $\ul{u}_\theta(\tilde{T})$ results in a tighter upper bound for $T_{\text{bu}}$. Note that while this extension may introduce a discontinuity in the derivative at the connection point $t = \tilde{T}$, the piecewise $C^1$ condition in Theorem~\ref{cor:ode-localsol} permits this irregularity.

\begin{table}[t]
\centering
\caption{
  Verification results for the blow-up time enclosure of the Riccati equation with the initial condition $u(0)=0$.
}
\label{table:Riccati}
\renewcommand{\arraystretch}{1.3}
\begin{tabular}{ccccc}
\hline
$\varepsilon$ & $c$ & $\underline{u}(\tilde{T})$ & Lower bound for $T_{\text{bu}}$ & Upper bound for $T_{\text{bu}}$ \\
\hline
$2^{-2}$ & 2.00 & 12.119 & 1.9375 & 2.020 \\
$2^{-3}$ & 2.00 & 14.029 & 1.9375 & 2.009 \\
$2^{-4}$ & 2.00 & 14.893 & 1.9375 & 2.005 \\
\hline
\end{tabular}
\renewcommand{\arraystretch}{1}
\end{table}

\begin{figure}[t]
    \centering
    \begin{tabular}{@{}cc@{}}
        \includegraphics[width=0.48\linewidth]{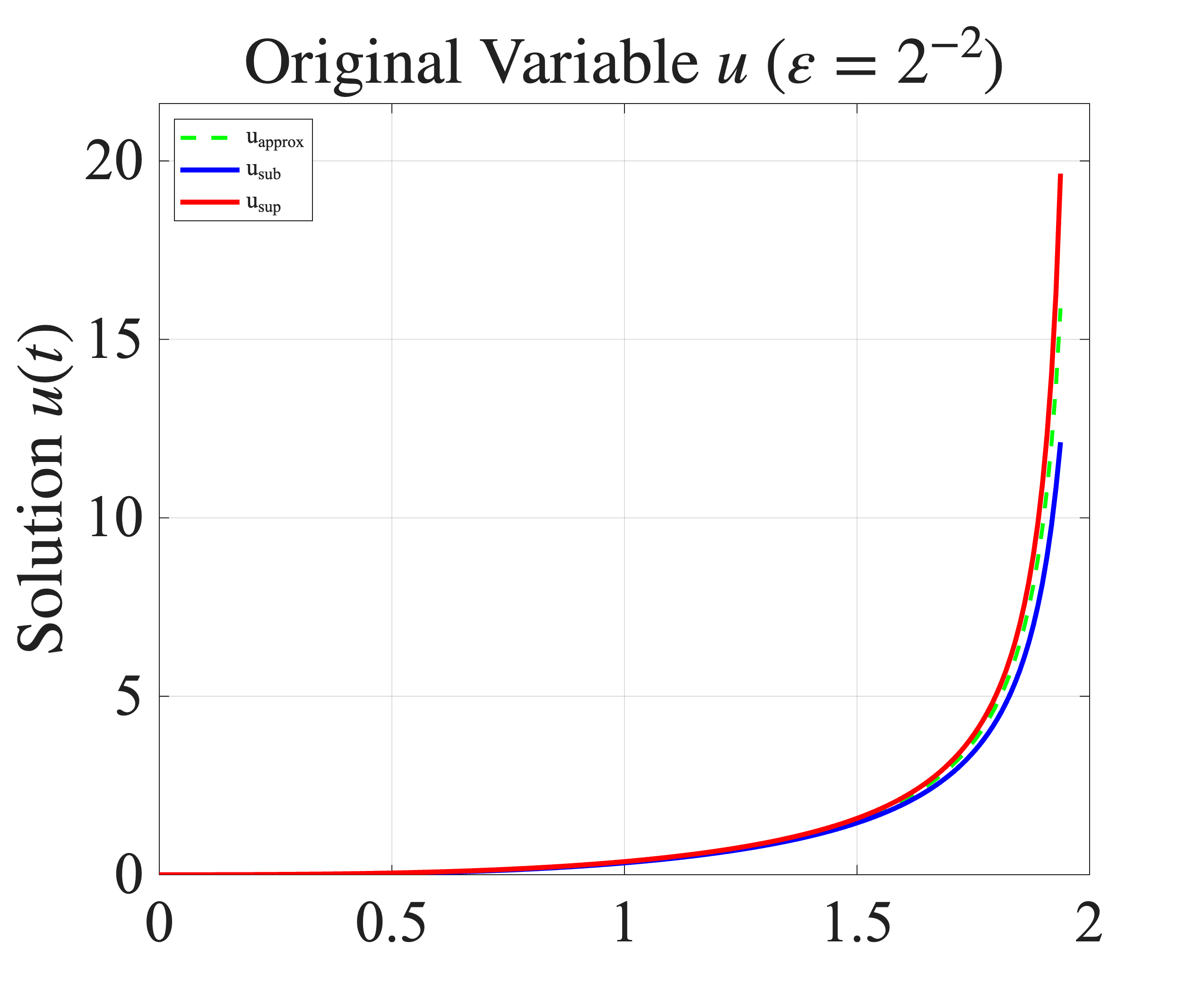} &
        \includegraphics[width=0.48\linewidth]{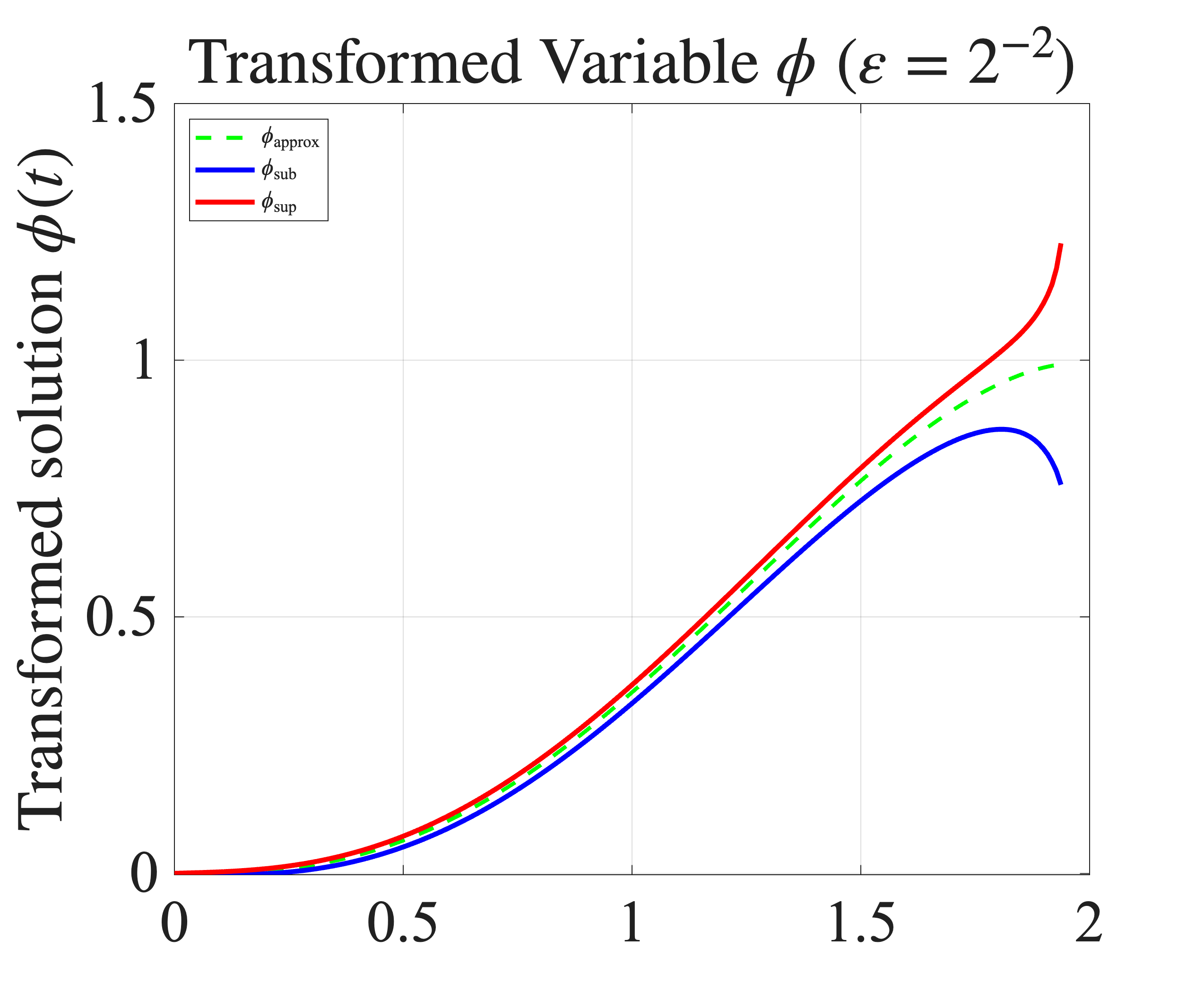} \\[0.5em]
        \includegraphics[width=0.48\linewidth]{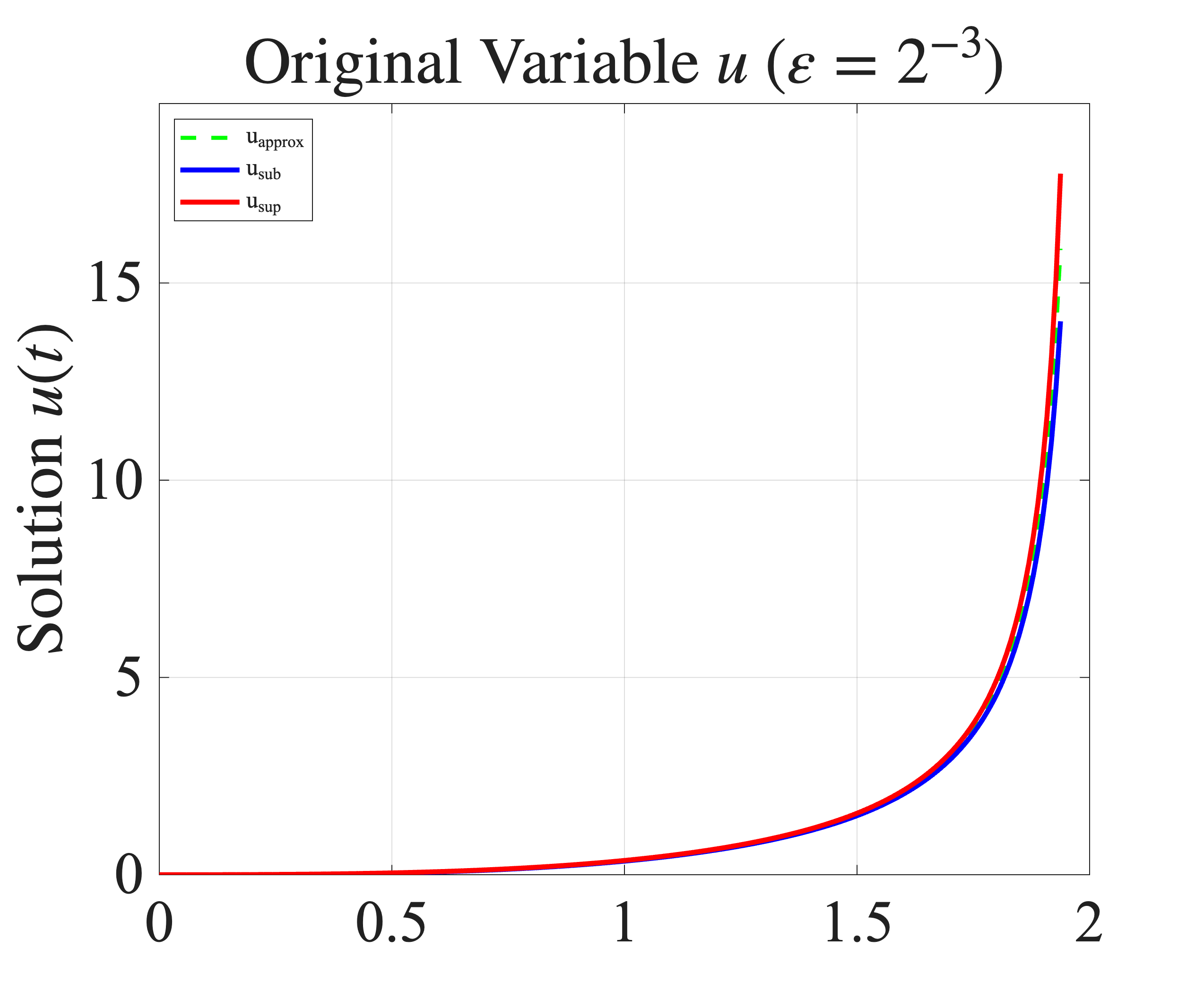} &
        \includegraphics[width=0.48\linewidth]{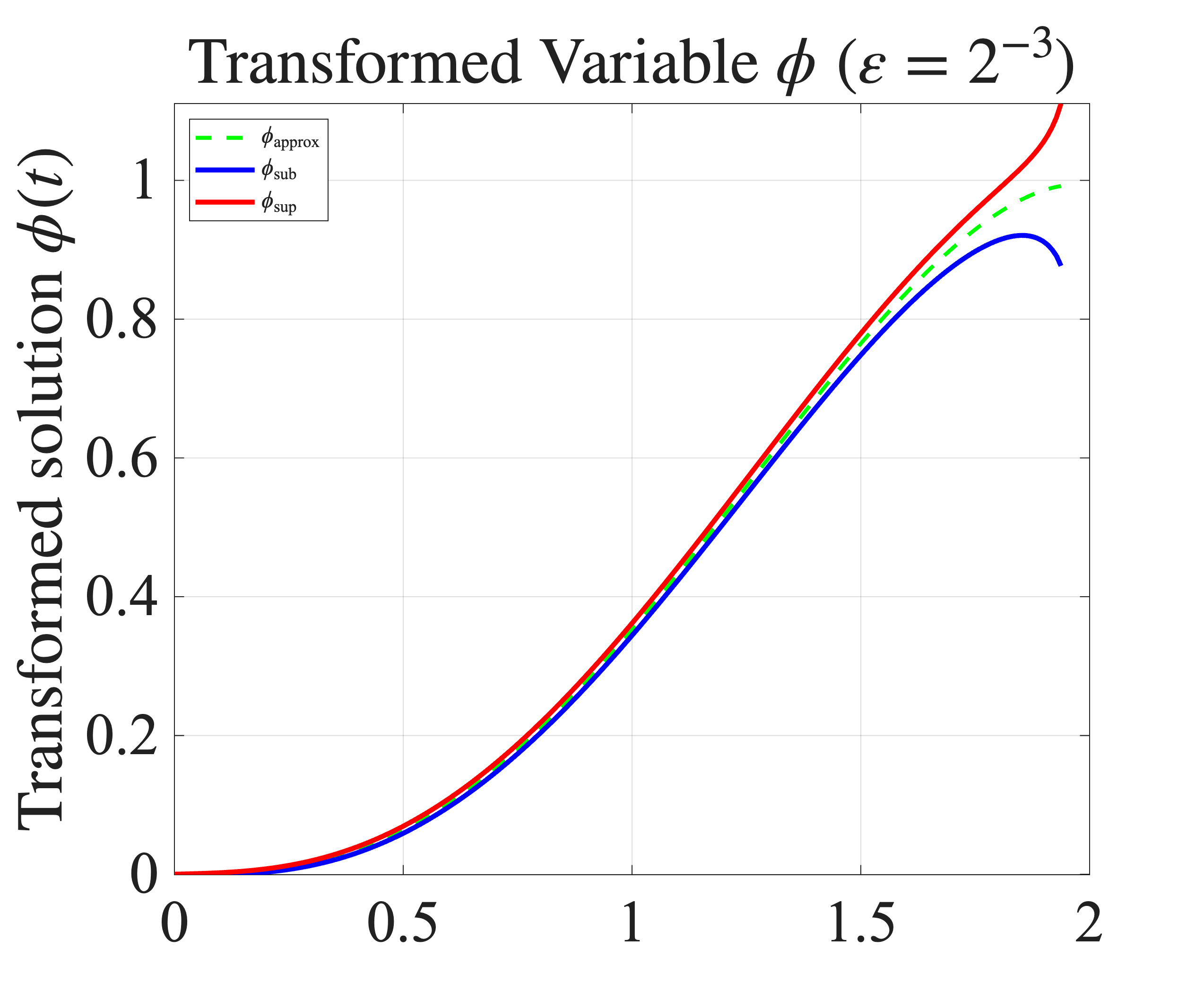} \\[0.5em]
        \includegraphics[width=0.48\linewidth]{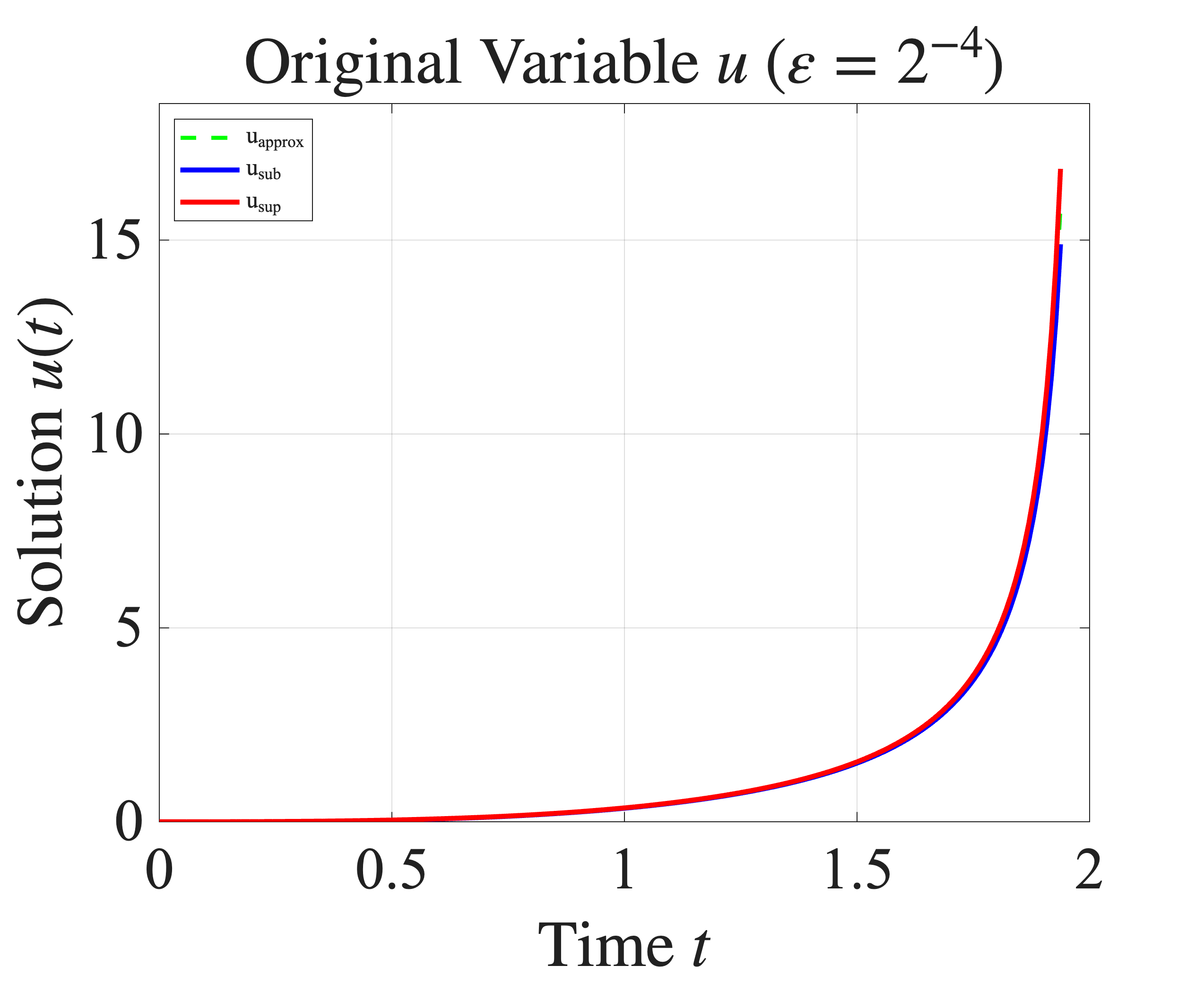} &
        \includegraphics[width=0.48\linewidth]{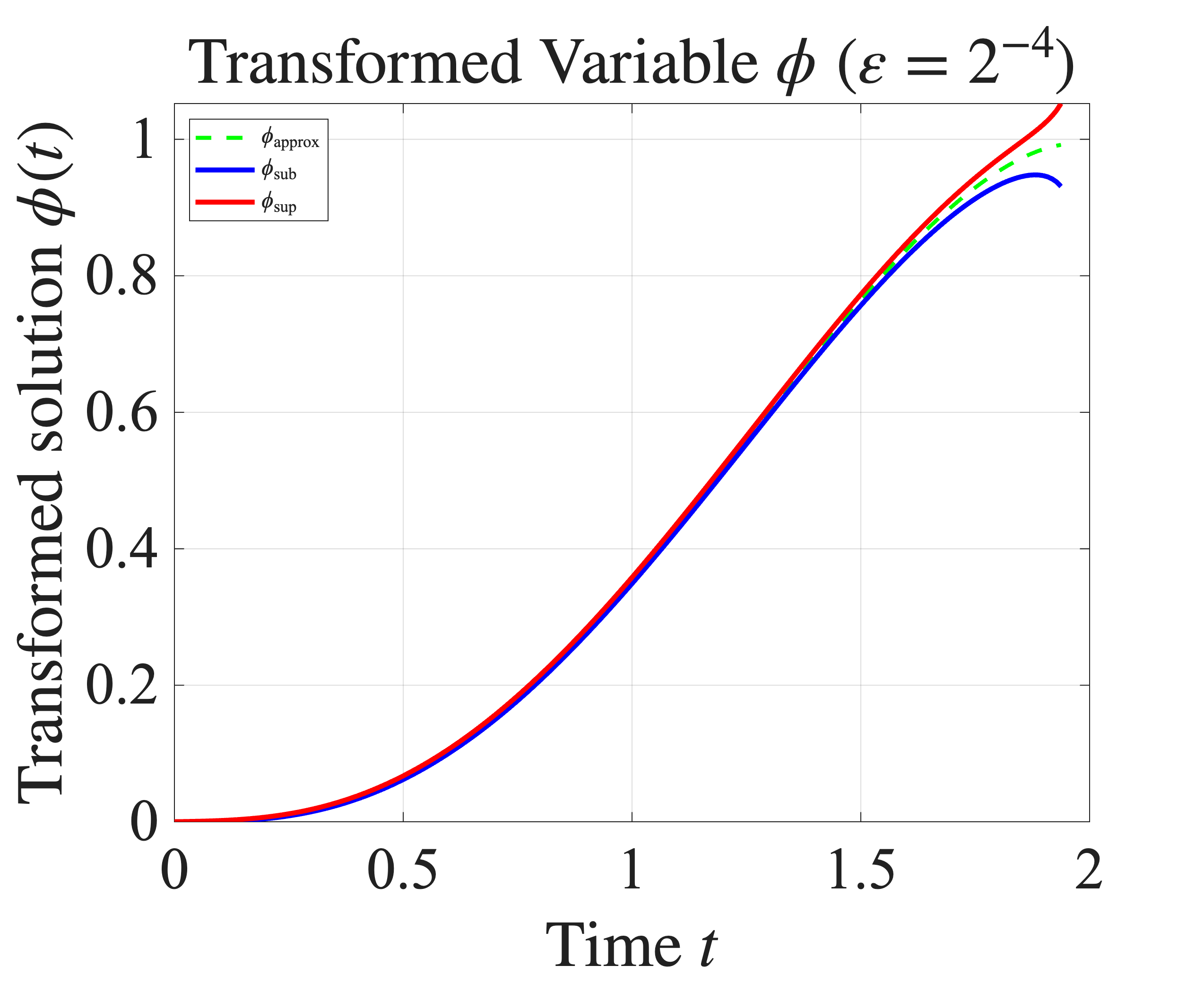}
    \end{tabular}
    \caption{Solution enclosures for original variable $u$ (left) and transformed variable $\phi$ (right) with $\varepsilon = 2^{-2}, 2^{-3}, 2^{-4}$ (from top to bottom).}
    \label{fig:Riccati_enclosure}
\end{figure}

Table~\ref{table:Riccati} summarizes the verification results for enclosing the blow-up time $T_{\text{bu}}$, while Fig.~\ref{fig:Riccati_enclosure} presents the solution enclosures for both the original variable $u$ and the transformed variable $\phi$ over the interval $t\in[0,\tilde{T}]$. As demonstrated in Table~\ref{table:Riccati}, decreasing the error tolerance $\varepsilon$ yields tighter enclosures of the blow-up time. This improvement is primarily driven by the value of $\underline{u}_\theta(\tilde{T})$, which increases as $\varepsilon$ decreases, thereby resulting in a stricter upper bound. Conversely, the lower bound remains constant because the parameter $\tilde{T}$ was fixed for these experiments. However, the rigorous verification provided by the Learn and Verify framework allows $\tilde{T}$ to be systematically adjusted to tighten the lower bound without compromising the validity of the results.

\section{Conclusions}\label{sec:conclusion}

In this paper, we have presented a novel ``Learn and Verify'' framework that addresses a fundamental limitation in the application of deep learning to differential equations: the absence of rigorous error bounds.
Focusing on ODEs, we proposed a two-stage methodology that incorporates specific model structures and tailored loss functions.
Crucially, our framework enables the systematic verification of numerical solutions without requiring a specific reference (e.g., a closed-form solution).
Our three numerical examples demonstrated the effectiveness of this mathematical verification for nonlinear ODEs.
Future work includes extending the proposed framework to a larger class of differential equations, as well as developing a scheme for the automatic tuning of parameters.

\appendices
\section{Property of DSM}\label{sec:proofDSM}

The DSM function defined in \eqref{dsm} exhibits the following relationship with the sample size $|\sample|$.
\begin{Prop}
  \label{prop:dsm_error_bound}
  Let $\mathcal{S}$ be a finite set of sampling points with cardinality $|\mathcal{S}|$.
  For any function $g: \mathcal{S} \to \mathbb{R}$, let
  \begin{align}
    \label{eq:dsm_max_violation}
    M = \max_{t \in \mathcal{S}} \max(g(t),0)
  \end{align}
  denote the true maximum positive violation.
  Then the DSM function defined in \eqref{dsm} satisfies
  \begin{align}
    \label{eq:dsm_bound}
    M < \DSM_{c_1,c_2}^{\sample}[g(t)]
      \le M + c_1 \log 2 + c_2 \log |\mathcal{S}|.
  \end{align}
\end{Prop}

\begin{proof}
  The DSM function can be viewed as a composition of an inner smooth approximation
  and an outer aggregation. Let $h(x) := c_1 \log(1 + \exp(x/c_1))$.
  Then for any $x \in \mathbb{R}$,
  \begin{align}
    \max(x, 0) < h(x) \le \max(x, 0) + c_1 \log 2.
    \label{eq:softplus_bound}
  \end{align}
  Substituting $h(g(t))$ into the definition, we can rewrite \eqref{dsm} as
  \[
    \DSM_{c_1,c_2}^{\sample}[g(t)]
     = c_2 \log \left( \sum_{t \in \mathcal{S}}
           \exp\left( \frac{h(g(t))}{c_2} \right) \right).
  \]
  For any set $\{y_t\}_{t\in\mathcal{S}}$ we have
  \begin{align}
    \max_{t} y_t
    \le c_2 \log \left( \sum_{t} \exp(y_t/c_2) \right)
    \le \max_{t} y_t + c_2 \log |\mathcal{S}|.
    \label{eq:lse_bound}
  \end{align}
  Setting $y_t = h(g(t))$ and combining \eqref{eq:softplus_bound} with
  \eqref{eq:lse_bound} result in
  \begin{align*}
    \DSM_{c_1,c_2}^{\sample}[g(t)]
    &\ge \max_{t \in \mathcal{S}} h(g(t))
      > \max_{t \in \mathcal{S}} (\max(g(t), 0)) = M,\\[0.3em]
    \DSM_{c_1,c_2}^{\sample}[g(t)]
    &\le \max_{t \in \mathcal{S}} h(g(t)) + c_2 \log |\mathcal{S}| \\
    &\le \max_{t \in \mathcal{S}} (\max(g(t), 0) + c_1 \log 2) + c_2 \log |\mathcal{S}| \\
    &= M + c_1 \log 2 + c_2 \log |\mathcal{S}|.
  \end{align*}
  This proves the claimed bound.
\end{proof}

The inequality \eqref{eq:dsm_bound} indicates that the parameter $c_2$ should scale with the logarithm of the sample size $|\mathcal{S}|$.
Consequently, when varying the sample size, adjusting $c_2$ according to this logarithmic scale is theoretically motivated and may enhance the stability and consistency of the method.
The determination of their optimal values is left for future work.

\section{Stable Implementation of DSM}\label{sec:stableDSM}

The direct evaluation of the DSM function defined in \eqref{dsm} is susceptible to numerical instability (i.e., arithmetic overflow) particularly when the parameter $c_1$ is small or when the input $g(t)$ takes large values. To mitigate these issues, we employ a stabilization technique (analogous to the Log-Sum-Exp trick) using the maximum positive violation $M$ defined in Proposition \ref{prop:dsm_error_bound}.
The procedure is as follows: using $M = \max_{t \in \mathcal{S}} \max(g(t), 0)$ from \eqref{eq:dsm_max_violation}, we reformulate the DSM function to factor out the dominant term:
\begin{align*}
    \DSM_{c_1,c_2} [g(t)] 
    &= c_2 \log \Bigg( \exp\left( \frac{M}{c_2} \right) \cdot \sum_{t \in \mathcal{S}} \Big( \exp\left(-\frac{M}{c_1}\right) \\
    &\qquad\qquad\qquad + \exp\left( \frac{g(t) - M}{c_1} \right) \Big)^{\frac{c_1}{c_2}} \Bigg) \\
    &= M + c_2 \log \Bigg( \sum_{t \in \mathcal{S}} \Big( \exp\left(-\frac{M}{c_1}\right) \\
    &\qquad\qquad\qquad + \exp\left( \frac{g(t) - M}{c_1} \right) \Big)^{\frac{c_1}{c_2}} \Bigg).
\end{align*}
By definition, $g(t) \leq M$; consequently, the exponential term $\exp((g(t)-M)/c_1)$ is always bounded in the interval $(0,1]$, ensuring the computation remains numerically stable regardless of the magnitude of the inputs.
This scaling operation is critical in practical implementations and allows for the use of small values for $c_1$ (e.g., $c_1=10^{-3}$), which are necessary to achieve a tight approximation of the maximum function for accurate sub- and super-solutions.

\section*{Acknowledgment}
This research was supported by JST Fusion Oriented Research for disruptive Science and Technology (FOREST) Program (Grant Number JPMJFR202S).

\ifCLASSOPTIONcaptionsoff
  \newpage
\fi

\bibliographystyle{IEEEtran}
\bibliography{ref.bib}

@book{Ladde1985,
  author    = {Ladde, G.S. and Lakshmikantham, V. and Vatsala, A.S.},
  title     = {Monotone Iterative Techniques for Nonlinear Differential Equations},
  publisher = {Pitman Advanced Publishing Program},
  address   = {London},
  year      = {1985}
}

@article{raissi2019physics,
  title     = {Physics-informed neural networks: A deep learning framework for solving forward and inverse problems involving nonlinear partial differential equations},
  author    = {Raissi, Maziar and Perdikaris, Paris and Karniadakis, George E},
  journal   = {J. Comput. Phys.},
  volume    = {378},
  pages     = {686--707},
  year      = {2019},
  publisher = {Elsevier}
}

@article{karniadakis2021physics,
  title     = {Physics-informed machine learning},
  author    = {Karniadakis, George Em and Kevrekidis, Ioannis G and Lu, Lu and Perdikaris, Paris and Wang, Sifan and Yang, Liu},
  journal   = {Nat. Rev. Phys.},
  volume    = {3},
  number    = {6},
  pages     = {422--440},
  year      = {2021},
  publisher = {Nature Publishing Group}
}

@article{cuomo2022scientific,
  title   = {Scientific Machine Learning through Physics-Informed Neural Networks: Where we are and What's next},
  author  = {Cuomo, Salvatore and Di Cola, Vincenzo Schiano and Giampaolo, Fabio and Rozza, Gianluigi and Raissi, Maziar and Piccialli, Francesco},
  journal = {arXiv:2201.05624},
  year    = {2022}
}

@inproceedings{sitzmann2020implicit,
  title     = {Implicit Neural Representations with Periodic Activation Functions},
  author    = {Sitzmann, Vincent and Martel, Julien N. P. and Bergman, Alexander W. and Lindell, David B. and Wetzstein, Gordon},
  booktitle = {Adv. Neural Inf. Process. Syst. (NeurIPS)},
  volume    = {33},
  pages     = {7462--7473},
  year      = {2020}
}

@misc{kashiwagi_kv,
  title        = {kv - {A} {C++} Library for Verified Numerical Computation},
  author       = {Kashiwagi, Masahide},
  year         = {2024},
  howpublished = {\url{http://verifiedby.me/kv/index-e.html}},
  note         = {Accessed: 2024}
}

@book{Klatte1993cxsc,
  title     = {{C-XSC}: A {C++} Class Library for Extended Scientific Computing},
  author    = {Klatte, Rudi and Kulisch, Ulrich and Lawo, Michael and Rauch, Michael and Wiethoff, Andreas},
  year      = {1993},
  publisher = {Springer-Verlag},
  address   = {Berlin, Heidelberg},
  doi       = {10.1007/978-3-642-58029-6}
}

@article{Nedialkov2010vnode,
  author    = {Nedialkov, Nedialko S.},
  title     = {{VNODE-LP}: A Validated Solver for Initial Value Problems in Ordinary Differential Equations},
  journal   = {ACM Trans. Math. Softw.},
  volume    = {37},
  number    = {3},
  pages     = {28:1--28:24},
  year      = {2010},
  publisher = {ACM},
  doi       = {10.1145/1824801.1824805}
}

@article{Kapela2011capd,
  author    = {Kapela, Tomasz and Mrozek, Marian and Wilczak, Daniel and Zgliczynski, Piotr},
  title     = {{CAPD::DynSys}: A Flexible {C++} Library for Rigorous Numerical Analysis of Dynamical Systems},
  journal   = {Commun. Nonlinear Sci. Numer. Simul.},
  volume    = {16},
  number    = {12},
  pages     = {4587--4607},
  year      = {2011},
  publisher = {Elsevier},
  doi       = {10.1016/j.cnsns.2011.02.012}
}

@article{revol2002mpfi,
  author    = {Revol, Nathalie and Rouillier, Fabrice},
  title     = {Motivations for an Arbitrary Precision Interval Arithmetic and the {MPFI} Library},
  journal   = {Reliab. Comput.},
  volume    = {11},
  number    = {4},
  pages     = {275--290},
  year      = {2005},
  publisher = {Springer},
  doi       = {10.1007/s11155-005-6891-y}
}

@article{rump1999intlab,
  title     = {{INTLAB - INTerval LABoratory}},
  author    = {Rump, Siegfried M},
  journal   = {Dev. Reliab. Comput.},
  pages     = {77--104},
  year      = {1999},
  publisher = {Kluwer Academic Publishers},
  address   = {Dordrecht},
  url       = {https://www.tuhh.de/ti3/rump/intlab/}
}

@book{neumaier1990interval,
  title     = {Interval methods for systems of equations},
  author    = {Neumaier, Arnold and Neumaier, Arnold},
  volume    = {37},
  year      = {1990},
  publisher = {Cambridge university press}
}

@book{moore2009introduction,
  title     = {Introduction to interval analysis},
  author    = {Moore, Ramon E and Kearfott, R Baker and Cloud, Michael J},
  year      = {2009},
  publisher = {SIAM}
}

@article{lohner1987enclosing,
  title     = {Enclosing the solutions of ordinary initial and boundary value problems},
  author    = {Lohner, Rudolf J},
  journal   = {Comput. Arith.},
  pages     = {255--286},
  year      = {1987},
  publisher = {Springer}
}

@article{kashiwagi1995power,
  title   = {Power Series Arithmetic and its Application to Numerical Validation},
  author  = {Kashiwagi, Masahide},
  journal = {Inst. Electron. Inf. Commun. Eng.},
  volume  = {95},
  number  = {296},
  pages   = {1--8},
  year    = {1995}
}

@article{nakao1988numerical,
  title     = {A numerical approach to the proof of existence of solutions for elliptic problems},
  author    = {Nakao, Mitsuhiro T},
  journal   = {Jpn. J. Appl. Math.},
  volume    = {5},
  number    = {2},
  pages     = {313--332},
  year      = {1988},
  publisher = {Springer}
}

@article{oishi1995numerical,
  title     = {Numerical verification of existence and inclusion of solutions for nonlinear operator equations},
  author    = {Oishi, Shin'ichi},
  journal   = {J. Comput. Appl. Math.},
  volume    = {60},
  number    = {1},
  pages     = {171--185},
  year      = {1995},
  publisher = {Elsevier}
}

@article{yamamoto1998numerical,
  title     = {A numerical verification method for solutions of boundary value problems with local uniqueness by Banach's fixed-point theorem},
  author    = {Yamamoto, Nobito},
  journal   = {SIAM J. Numer. Anal.},
  volume    = {35},
  number    = {5},
  pages     = {2004--2013},
  year      = {1998},
  publisher = {SIAM}
}

@article{day2007validated,
  title     = {Validated continuation for equilibria of {PDEs}},
  author    = {Day, Sarah and Lessard, Jean-Philippe and Mischaikow, Konstantin},
  journal   = {SIAM J. Numer. Anal.},
  volume    = {45},
  number    = {4},
  pages     = {1398--1424},
  year      = {2007},
  publisher = {SIAM}
}

@article{plum1991computer,
  title     = {Computer-assisted existence proofs for two-point boundary value problems},
  author    = {Plum, Michael},
  journal   = {Computing},
  volume    = {46},
  number    = {1},
  pages     = {19--34},
  year      = {1991},
  publisher = {Springer}
}

@article{takayasu2013verified,
  title     = {Verified computations to semilinear elliptic boundary value problems on arbitrary polygonal domains},
  author    = {Takayasu, Akitoshi and Liu, Xuefeng and Oishi, Shin'ichi},
  journal   = {Nonlinear Theory Appl.},
  volume    = {4},
  number    = {1},
  pages     = {34--61},
  year      = {2013},
  publisher = {The Institute of Electronics, Information and Communication Engineers}
}

@article{shin2020convergence,
  title   = {On the convergence of PINNs for linear second-order elliptic and parabolic type PDEs},
  author  = {Shin, Younghak and Darbon, J{\'e}r{\^o}me and Karniadakis, George E},
  journal = {arXiv:2004.01806},
  year    = {2020}
}

@article{molinaro2020pinn,
  title   = {Estimates on the generalization error of {PINNs} for approximating {PDE}s},
  author  = {Molinaro, Roberto and Mishra, Siddhartha},
  journal = {arXiv:2007.01138},
  year    = {2020}
}

@article{bonito2024consistent,
  title   = {Convergence and error control of consistent physics informed neural networks for elliptic PDEs},
  author  = {Bonito, Andrea and DeVore, Ronald and Petrova, Guergana and Siegel, John W},
  journal = {arXiv:2406.09217},
  year    = {2024}
}

@article{jiao2022rate,
  title   = {A rate of convergence of physics informed neural networks for the linear second order elliptic PDEs},
  author  = {Jiao, Yuling and others},
  journal = {arXiv:2109.01780},
  year    = {2022}
}

@article{masri2025unified,
  title   = {A unified framework for the error analysis of physics-informed neural networks},
  author  = {Masri, Rami and Mardal, Kent-Andr{\'e} and Zeinhofer, Michael},
  journal = {IMA J. Numer. Anal.},
  year    = {2025}
}

@article{sirignano2018dgm,
  title   = {DGM: A deep learning algorithm for solving partial differential equations},
  author  = {Sirignano, Justin and Spiliopoulos, Konstantinos},
  journal = {J. Comput. Phys.},
  volume  = {375},
  pages   = {1339--1364},
  year    = {2018}
}

@article{e2018deepritz,
  title   = {The Deep Ritz Method: A deep learning-based numerical algorithm for solving variational problems},
  author  = {E, Weinan and Yu, Bing},
  journal = {Commun. Math. Stat.},
  volume  = {6},
  number  = {1},
  pages   = {1--12},
  year    = {2018}
}

@article{e2017deepbsde,
  title   = {Deep learning-based numerical methods for high-dimensional parabolic partial differential equations and backward stochastic differential equations},
  author  = {E, Weinan and Han, Jiequn and Jentzen, Arnulf},
  journal = {Commun. Math. Stat.},
  volume  = {5},
  number  = {4},
  pages   = {349--380},
  year    = {2017}
}

@article{chassagneux2020deepbsde,
  title   = {Convergence of the deep BSDE method for coupled FBSDEs},
  author  = {Chassagneux, Jean-Fran{\c{c}}ois and Crisan, Dan and Delarue, Fran{\c{c}}ois},
  journal = {Probab. Uncertain. Quant. Risk},
  volume  = {5},
  number  = {1},
  pages   = {1--34},
  year    = {2020}
}

@article{lu2021deeponet,
  title   = {Learning nonlinear operators via {DeepONet} based on the universal approximation theorem of operators},
  author  = {Lu, Lu and Jin, Pengzhan and Pang, Guofei and Zhang, Zongyi and Karniadakis, George Em},
  journal = {Nat. Mach. Intell.},
  volume  = {3},
  number  = {3},
  pages   = {218--229},
  year    = {2021}
}

@inproceedings{li2021fno,
  title     = {Fourier Neural Operator for Parametric Partial Differential Equations},
  author    = {Li, Zongyi and Kovachki, Nikola and Azizzadenesheli, Kamyar and Liu, Kaushik and Bhattacharya, Kaushik and Stuart, Andrew and Anandkumar, Anima},
  booktitle = {Int. Conf. Learn. Represent. (ICLR)},
  year      = {2021}
}

@inproceedings{hillebrecht2022certified,
  title     = {Certified machine learning: A posteriori error estimation for physics-informed neural networks},
  author    = {Hillebrecht, Birgit Simone and Unger, Benjamin},
  booktitle = {Int. Jt. Conf. Neural Netw. (IJCNN)},
  pages     = {1--8},
  year      = {2022},
  publisher = {IEEE},
  doi       = {10.1109/IJCNN55064.2022.9892569}
}

@inproceedings{liu2023residual,
  title     = {Residual-based error bound for physics-informed neural networks},
  author    = {Liu, Shuheng and Huang, Xiyue and Protopapas, Pavlos},
  booktitle = {Proc. Conf. Uncertain. Artif. Intell.},
  publisher = {PMLR},
  pages     = {1284--1293},
  year      = {2023},
  volume    = {216}
}

@inproceedings{ahmadova2025lower,
  title     = {Certifying Physics-Informed Neural Networks through Lower Error Bounds},
  author    = {Ahmadova, Arzu and Huseynov, Ismail and Bashirov, Agamirza},
  booktitle = {EurIPS Workshop Differ. Syst. (DiffSys)},
  year      = {2025}
}

\end{document}